\documentclass[letterpaper]{article}
\usepackage[top=1.5in, bottom=1.5in, left=1.20in, right=1.20in]{geometry}

\PassOptionsToPackage{numbers}{natbib}

\pdfoutput=1
\usepackage{hyperref}       
\usepackage{url}            
\usepackage{booktabs}       
\usepackage{amsfonts}       
\usepackage{nicefrac}       
\usepackage{microtype}      
\usepackage{natbib}

\usepackage{url}

\usepackage{dsfont} 
\usepackage{xspace, amsmath, amssymb, amsthm, bbm, bm}
\usepackage{color}
\usepackage{fancyvrb}

\usepackage{pdfsync}
\pdfoutput=1

\usepackage{graphicx}
\usepackage{tabularx}

\usepackage{algorithm}
\usepackage{algorithmic}

\usepackage{multirow}
\usepackage{hhline}

\usepackage{booktabs}

\usepackage{capt-of}

\newtheorem{lem}{Lemma}
\newtheorem{thm}[lem]{Theorem}
\newtheorem{cor}[lem]{Corollary}
\newtheorem{defn}[lem]{Definition}

\def\bfu{{{\mathbf u}}}

\def\eps{\ensuremath{\epsilon}\xspace}

\allowdisplaybreaks 

\def\one{\ensuremath{\mathds{1}\xspace}} 
 
\def\larrow{\ensuremath{\leftarrow}\xspace} 
\def\rarrow{\ensuremath{\rightarrow}\xspace}

\def\T{\ensuremath{\top}}  
\def\sig{\ensuremath{\sigma}\xspace}

\def\eps{\ensuremath{\epsilon}\xspace}

\newtheorem{assump}{Assumption}

\def\bfpi{{\ensuremath{\bm{\pi}}\xspace}}

\def\dt{{\ensuremath{\delta}\xspace} }
\def\sm{{\ensuremath{\setminus}\xspace} }

\def\sig{\ensuremath{\sigma}\xspace}

\usepackage[export]{adjustbox}

\def\lfl{\lfloor} 
\def\rfl{\rfloor}

\usepackage{pbox} 

\newcommand{\fr}[2]{ { \frac{#1}{#2} }}
\def\lt{\left}
\def\rt{\right}
\def\gam{{\ensuremath{\gamma}\xspace} }

\makeatletter
\newcommand{\vast}{\bBigg@{3}}
\newcommand{\Vast}{\bBigg@{4}}
\makeatother

\def\bfl{{{\boldsymbol \ell}}}
\def\dsR{{{\mathds{R}}}}

\def\w{{{\mathbf w}}}
\def\x{{{\mathbf x}}}

\def\z{{{\mathbf z}}}
\def\la{{\langle}}
\def\ra{{\rangle}}

\def\calL{\ensuremath{\mathcal{L}}\xspace}

\usepackage{pifont}
\newcommand{\cmark}{\ding{51}}%
\newcommand{\gyes}{{\color[rgb]{0,.8,0}\cmark}}
\newcommand{\xmark}{\ding{55}}%
\newcommand{\rno}{{\color[rgb]{.8,0,0}\xmark}}

\def\bfell{{{\boldsymbol\ell}}}
\def\p{{{\mathbf p}}}
 
\def\cD{\ensuremath{\mathcal{D}}\xspace} 
\def\cW{\ensuremath{\mathcal{W}}\xspace} 
 
\def\e{{{\mathbf e}}}

\def\w{{{\mathbf w}}}

\def\KL{\ensuremath{\text{\normalfont{KL}}}\xspace}  
\def\cB{\ensuremath{\mathcal{B}}\xspace} 
\def\cN{\ensuremath{\mathcal{N}}\xspace} 
\def\cI{\ensuremath{\mathcal{I}}\xspace} 
\def\cA{\ensuremath{\mathcal{A}}\xspace} 
\def\cJ{\ensuremath{\mathcal{J}}\xspace} 
\def\cM{\ensuremath{\mathcal{M}}\xspace}

\def\polylog{\ensuremath{\text{\normalfont{polylog}}}}  

\newcommand{\blue}[1]{{\color[rgb]{.3,.3,1}#1}}
\usepackage{enumitem}

\def\hatp{\ensuremath{\widehat p}}
\def\hatbfp{\ensuremath{\widehat{\mathbf{p}}}}

\def\bfcI{\ensuremath{\boldsymbol{\mathcal{I}}}} 
\def\Active{\ensuremath{\text{\normalfont Active}}}

\def\CBCE{\ensuremath{\text{\normalfont CBCE}}}
\def\Regret{\ensuremath{\text{\normalfont Regret}}}
\def\Wealth{\ensuremath{\text{\normalfont Wealth}}}
\def\CB{\ensuremath{\text{\normalfont CB}}}
\def\SARegret{\ensuremath{\text{\normalfont SA-Regret}}}
\def\OGD{\ensuremath{\text{\normalfont OGD}}}

\def\M{\ensuremath{\mathbf{M}}} 
\def\dsN{{{\mathds{N}}}}

\def\bfu{{{\mathbf u}}}
\def\bfl{{{\boldsymbol \ell}}}

\def\barG{{\overline{G}}}
\def\barZ{{\overline{Z}}}
\def\tilL{\ensuremath{\widetilde{L}}}
\def\cP{\ensuremath{\mathcal{P}}\xspace} 

\usepackage{mathtools}
\usepackage[T1]{fontenc}
\usepackage[utf8]{inputenc}
\usepackage{lmodern}
\usepackage{anyfontsize}

\usepackage{wrapfig}
\usepackage{lipsum}
\usepackage{microtype}      

\newcommand{\revision}[1]{{#1}}

\title{Online Learning for Changing Environments using Coin Betting}

\author{
  \begin{tabular}{cc}
  Kwang-Sung Jun & Francesco Orabona\\
  UW-Madison & Stony Brook University \\
  \texttt{kjun@discovery.wisc.edu} & \texttt{francesco@orabona.com} \\
   & \\
  Stephen Wright & Rebecca Willett\\
  UW-Madison & UW-Madison\\
  \texttt{swright@cs.wisc.edu} & \texttt{willett@discovery.wisc.edu} \\
  \end{tabular}
}

\begin{document}

\setlength{\abovedisplayskip}{5pt}
\setlength{\belowdisplayskip}{4pt}
\setlength{\abovedisplayshortskip}{5pt}
\setlength{\belowdisplayshortskip}{4pt}



\date{}
\maketitle

\begin{abstract}
A key challenge in online learning is that classical algorithms can be slow to adapt to changing environments.
Recent studies have proposed ``meta'' algorithms that convert any online learning algorithm to one that is adaptive to changing environments, where the adaptivity is analyzed in a quantity called the strongly-adaptive regret.
This paper describes a new meta algorithm that has a strongly-adaptive regret bound that is a factor of $\sqrt{\log(T)}$ better than other algorithms with the same time complexity, where $T$ is the time horizon.
\revision{ We also extend our algorithm to achieve a first-order (i.e., dependent on the observed losses) strongly-adaptive regret bound for the first time, to our knowledge.}
At its heart is a new parameter-free algorithm for the learning with expert advice (LEA) problem in which experts sometimes do not output advice for consecutive time steps (i.e., \emph{sleeping} experts).
This algorithm is derived by a reduction from optimal algorithms for the so-called coin betting problem.
Empirical results show that our algorithm outperforms state-of-the-art methods in both learning with expert advice and metric learning scenarios.
\end{abstract}

\section{Introduction}

Online learning algorithms are typically tailored to stationary
environments, but in many applications the environment is dynamic. In
online portfolio management, for example, stock price trends can vary
unexpectedly, and the ability to track changing trends and adapt to
them are crucial in maximizing profit. In product reviews, words
describing product quality may change over time as products evolve and
the tastes of customers change.  Keeping track of the changes in the
metric describing the relationship between review text and rating is
crucial for improving analysis and the quality of recommendations.

We consider the problem of adapting to changing environments in the
online learning context.  Let $\cD$ be the decision space, $\calL$ be
a family of loss functions that map $\cD$ to $\dsR$, and $T$ be the
target time horizon.  Let $\cA$ be an online learning algorithm.  We
define the online learning protocol in Figure~\ref{fig:ol}.

\begin{figure}[h]
  \begin{center}
\fbox{\begin{minipage}{223pt}
At each time $t = 1,2,\ldots,T$,
\begin{itemize}[topsep=2pt,itemsep=0ex,partopsep=1ex,parsep=1ex]
  \item The learner $\cA$ picks a decision $\x^{\cA}_t \in \cD$.
  \item The environment reveals a loss function $f_t \in \calL$.
  \item The learner $\cA$ suffers loss $f_t(\x^{\cA}_t)$.
\end{itemize}
\end{minipage}}
\end{center}
\vspace{-15pt}
\caption{Online learning protocol}
\label{fig:ol}
\vspace{-5pt}
\end{figure}
The usual goal of online learning is to find a strategy that compares
favorably with the best fixed comparator in a subset $\cW$ of decision
space $\cD$, in hindsight. (Often, $\cW = \cD$.)  Classically, one
seeks a low value of the following (cumulative) \emph{static} regret
objective:
\[
  \mbox{Regret}^\cA_T := \sum_{t=1}^T f_t(\x^\cA_t) - \min_{\w \in \cW} \sum_{t=1}^T f_t(\w) \;.
\]
When the environment is changing, static regret is not a suitable measure, since it compares the learning strategy against a decision that is fixed for all $t$.
We need to make use of stronger notions of regret that allow comparators to change over time.
We introduce the notation $[T] := \{1, \ldots,T\}$ and $[A..B] := \{A, A+1, \ldots, B\}$.
\citet{daniely15strongly} defined {\em strongly-adaptive regret
  (SA-Regret)}, which measures the regret over \emph{any} time interval $I = [I_1..I_2] \subseteq [T]$:
\begin{align}\label{sa-regret}                                  
\mbox{SA-Regret}_T^\cA(I) := \lt( \sum_{t\in I} f_t(\x^\cA_t) -
\min_{\w\in\cW} \sum_{t\in I} f_t(\w) \rt) \;.
\end{align}                                                                                  
Throughout, $I_1$ ($I_2$) denotes the starting (ending) time step of
an interval $I$.  \revision{ We call an algorithm \emph{strongly-adaptive} if it has a \emph{low} value of SA-Regret, by which we
  mean a value $O(\text{polylog}(T) R_{\cP}(I))$, where $R_{\cP}(I)$
  is the minimax static regret of the online learning problem $\cP$
  restricted to interval $I$.}

Let us call $\w_{1:T} := \{\w_1, \ldots, \w_T\}$ an \emph{$m$-shift
  sequence} if it changes at most $m$ times, that is,
$\sum_{j=1}^{T-1} \one\{\w_j \neq \w_{j+1}\} \le m$.  A related notion,
$m$-shift regret~\cite{herbster98tracking}, measures the regret with respect to a comparator that changes at most $m$ times in $T$ time steps.
\vspace{-3pt}
\begin{align*}
  m\mbox{-Shift-Regret}^\cA_T := \sum_{t=1}^T f_t(\x^\cA_t) - \min_{\w_{1:T} \in \cW^{T} \;:\; m\mbox{-shift seq.}} \sum_{t=1}^{T} f_t(\w_t) \;.
\end{align*}
While the $m$-shift regret is more interpretable, SA-Regret is a
stronger notion since \revision{it is well-known} that a tight
SA-Regret bound implies a tight $m$-shift regret
bound~\cite{luo15achieving,daniely15strongly}, as we discuss further
in Section~\ref{sec:saregret}.  \revision{As noted
  by~\cite{zhang17strongly}, SA-Regret has a strong connection to
  so-called dynamic regret (with respect to the temporal variations of the $f_t$'s).  }

\begin{table}
  {\footnotesize\centering
\begin{tabular}{|c|c|c|c|} \hline
Algorithm & SA-Regret order                          & Time factor \\ \hline
FLH~\cite{hazan07adaptive}    & $ \sqrt{T \log T}$                 & $T$         \\  
AFLH~\cite{hazan07adaptive}   & $ \sqrt{T \log T } \log (I_2-I_1)$ & $\log T $   \\  \hline
SAOL~\cite{daniely15strongly} & $ \sqrt{(I_2 - I_1)\log^2 (I_2)}$  & $\log T $   \\  \hline
$\CBCE$ (ours)                         & $ \sqrt{(I_2-I_1)\log (I_2)}$      & $\log T $ \\  \hline
\end{tabular}
\caption{SA-Regret bounds of meta algorithms on $I = [I_1..I_2]
  \subseteq [T]$.  We show the part of the regret due to the meta
  algorithm only, not the black-box.  The last column is the
  multiplicative factor in the time complexity introduced by the meta
  algorithm. CBCE (our algorithm) achieves the best SA-Regret and time
  complexity.}
\label{tab:regret-meta-sa}
}
\vspace{-18pt}
\end{table}

Several generic online algorithms that adapt to changing environments have been proposed recently.
Rather than being designed for a specific learning problem, these are ``meta'' algorithms that take \emph{any} online learning algorithm as a black-box and turn it into an adaptive one.
We summarize the SA-Regret of existing meta algorithms in Table~\ref{tab:regret-meta-sa}.
In particular, the pioneering work of~\citet{hazan07adaptive} introduced \emph{adaptive regret}, a slightly weaker notion than the SA-Regret, and proposed two meta algorithms called Follow-the-Leading-History (FLH) and Advanced FLH (AFLH).\footnote{%
Strongly adaptive regret is similar to the notion of adaptive regret introduced by~\cite{hazan07adaptive}, but emphasizes the dependency on the interval length $|I|$.
}
However, their SA-Regret depends on $T$ rather than $|I|$ and hence can be significantly larger.
The SAOL approach of~\cite{daniely15strongly} improves the SA-Regret to $O\lt(\sqrt{(I_2-I_1)\log^2(I_2)}\rt)$.

\begin{table*}
     {\centering\footnotesize
\begin{tabular}{|c|c|c|c|c|c|c|} \hline
Algorithm   & $m$-shift regret & Running time               & Agnostic to $m$ \\ \hline
Fixed Share~\cite{herbster98tracking,cesa-bianchi12mirror} & $\sqrt{mT(\log N + \log T)}$ & $NT$     & \rno           \\
           & $\sqrt{m^2 T(\log N + \log T)}$ & $NT$ & \gyes           \\ \hline
GeneralTracking$\la$EXP$\ra$~\cite{gyorgy12efficient} & $\sqrt{mT(\log N + m\log^2 T)}$ & $NT\log T$&\gyes\\
  & $\sqrt{mT(\log N + \log^2 T)}$ & $NT\log T$ & \rno\\ 
($\gam\in(0,1)$) & $\sqrt{\fr{1}{\gam}mT(\log N + m\log T)}$ & $NT^{1+\gam}\log T$ & \gyes\\ 
                 & $\sqrt{\fr{1}{\gam}mT(\log N + \log T)}$ & $NT^{1+\gam}\log T$ &\rno\\ \hline
ATV~\cite{luo15achieving} & $\sqrt{mT(\log N + \log T)}$ & $NT^2$ & \gyes       \\ \hline
SAOL$\la$MW$\ra$~\cite{daniely15strongly} & $\sqrt{mT(\log N + \log^2 T)}$ & $NT\log T$        & \gyes       \\ \hline
\CBCE$\la\CB\ra$ (ours)      & $\sqrt{mT(\log N + \log T)}$ & $NT\log T$   & \gyes       \\ \hline
\end{tabular}
\caption{$m$-shift regret bounds of LEA algorithms.  Our proposed
  algorithm (last line) achieves the best regret among those with the
  same time complexity and does not need to know $m$.  Each quantity
  omits constant factors.  ``Agnostic to $m$'' means that an algorithm
  does not need to know the number $m$ of switches in the best expert.}
\label{tab:tracking}
}
\vspace{-18pt}
\end{table*} 

In this paper, we propose a new meta algorithm called {\em Coin
  Betting for Changing Environments} ($\CBCE$) that combines the idea
of ``sleeping experts'' introduced in
\cite{blum97empirical,freund97using} with the Coin Betting (CB)
algorithm~\cite{orabona16from}.  The SA-Regret of CBCE is better by a
factor $\sqrt{\log(I_2)}$ than that of SAOL, as shown in
Table~\ref{tab:regret-meta-sa}.  We present our extension of CB to
sleeping experts and prove its regret bound in
Section~\ref{sec:sleeping}. This result leads to the improved
SA-Regret bound of CBCE in Section~\ref{sec:cbce}.

Our improved bound yields a number of improvements in various online
learning problems. In describing these improvements, we designate by
$\cM\la\cB\ra$ a complete algorithm assembled from meta algorithm
$\cM$ and black-box $\cB$. In this notation, our algorithm is
designated by $\CBCE\la\CB\ra$.

Consider the learning with expert advice (LEA) problem with $N$
experts.  We make comparisons with respect to $m$-shift regret bounds, as many
LEA algorithms provide only bounds of this type. Our algorithm
$\CBCE\la\CB\ra$ has $m$-shift regret
\begin{align*}
  O \lt(\sqrt{mT(\log N + \log T)} \rt)
\end{align*}
and time complexity $O(NT \log T)$.  
This regret is a factor $\sqrt{\log T}$ better than existing algorithms with the same time complexity.  
Although AdaNormalHedge.TV (ATV) and Fixed Share achieve the same regret, the former has larger time complexity, and the latter requires prior knowledge of the number of shifts $m$.
We summarize the $m$-shift regret bounds of various algorithms in Table~\ref{tab:tracking}.
\revision{
We emphasize that the same regret order and time complexity as CBCE$\la$CB$\ra$ can be achieved by combining our proposed CBCE meta algorithm with any blackbox algorithm (e.g., AdaNormalHedge~\cite{luo15achieving}).
}

In online convex optimization with $G$-Lipschitz loss functions over a
convex set $D \in \dsR^d$ of diameter $B$, Online Gradient Descent
(OGD) has regret $O(BG\sqrt{T})$~\cite{ss12online}.  Thus, $\CBCE$
with OGD ($\CBCE\la\text{OGD}\ra$) has the following SA-Regret:
\begin{equation*} \begin{aligned}
O\lt((BG + \sqrt{\log(I_2)})\sqrt{|I|}\rt),
\end{aligned} \end{equation*}
which improves by a factor $\sqrt{\log(I_2)}$ over SAOL$\la$OGD$\ra$.

\revision{
We also propose an improved version of CBCE that has a so-called \emph{first-order} regret bound.
That is, the SA-Regret on an interval $I=\{I_1,\ldots,I_2\}$ scales with $\min_{\w\in\cW} \sum_{t\in I} f_t(\w)$ rather than $|I|$ as follows:
\begin{align*}
  O\lt(\log(I_2) \sqrt{\min_{\w\in\cW} \sum_{t\in I} f_t(\w)} + \polylog(I_2)\rt) \;,
\end{align*}
where we omit the term due to the blackbox algorithm.
We emphasize that, while there is an extra $\sqrt{\log(I_2)}$ factor and additive term, the main quantity $\min_{\w\in\cW} \sum_{t\in [T]} f_t(\w)$ can be significantly smaller than $|I|$ if there exists the decision $\w$ whose loss is very small in $I$.
To our knowledge, this is the first first-order SA-Regret bound in online learning.\footnote{%
  First-order bounds are available for specific online learning problems.
  For LEA, for example, AdaNormalHedge.TV~\cite{luo15achieving} has a first-order regret bound.
}
}


In Section~\ref{sec:expr}, we compare $\CBCE$ empirically to a number
of meta algorithms for changing environments in two online learning
problems: LEA and Mahalanobis metric learning.  We observe that
$\CBCE$ outperforms the state-of-the-art methods in both tasks, thus
confirming our theoretical findings.

\section{Meta Algorithms for Changing Environments}
\label{sec:meta}

Let $\cB$ be a black-box online learning algorithm following the
protocol in Figure~\ref{fig:ol}.  A trick commonly used in designing a
meta algorithm $\cM$ for changing environments is to initiate a new
instance of $\cB$ at every time
step~\cite{hazan07adaptive,gyorgy12efficient,adamskiy12acloser}.  That
is, we run $\cB$ independently for each interval $J$ in $\{[t..\infty]
\mid t = 1,2,\ldots\}$.  Denoting by $\cB_J$ the run of black-box
$\cB$ on interval $J$, a meta algorithm at time $t$ takes a weighted
average of decisions from the runs $\{\cB_J \, : \, t \in J \}$.  The
underlying idea is as follows.  Suppose we are at time $t$ and the
environment has changed at an earlier time $t' < t$. We hope that the
meta algorithm would assign a large weight to the black-box run
$\cB_{J'}$ (where $J' = [t'..\infty]$), since other runs are either
based on data prior to $t'$ or use only a subset of the data generated
since $t'$.  Ideally, the meta algorithm would assign a large weight
to $\cB_{J'}$ {\em soon after time $t'$}, by carefully examining the
online performance of each black-box run.

This schema requires updating of $t$ instances of the black-box
algorithm at each time step $t$, leading to a $O(t)$ multiplicative
increase in complexity over a single run.  This factor can be reduced
to $O(\log t)$ by restarting black-box algorithms on a carefully
designed set of intervals, such as the geometric covering
intervals~\cite{daniely15strongly} (GC) or the data streaming
intervals~\cite{hazan07adaptive,gyorgy12efficient} (DS), which is a
special case of a more general set of intervals considered
in~\cite{veness13partition}.  While both GC and DS achieve the same
goal, as we show in Section~\ref{sec:ds},\footnote{Except for a subtle
  case, which we discuss in Section~\ref{sec:subtle}.  } we use the
former as our starting point for ease of exposition.

\paragraph{Geometric Covering Intervals.}

We define the $\cJ_k$ to be the collection of intervals of length $2^k$:
\[
 \cJ_k := \left\{ [ \lt(i\cdot 2^k\rt) .. \lt((i+1)\cdot2^k - 1\rt) ]: i \in
 \mathds{N} \right\}, \quad \forall k \in \{0,1,\ldots\}.
\]
The geometric covering intervals~\cite{daniely15strongly} are
\begin{align*}
  \cJ := \bigcup_{k\in\{0,1,\ldots\}} \cJ_k \;.
\end{align*}
That is, $\cJ$ is the set of intervals of doubling length, with
intervals of size $2^k$ exactly partitioning the set $\mathds{N} \sm
\{1,\ldots,2^k-1\}$; see Figure~\ref{fig:gc}.

Define the set of intervals that includes time $t$ as follows:
\[
  \mbox{Active}(t) := \{J\in\cJ: t \in J\}\;.
\]
It can be shown that $|\mbox{Active}(t)| = \lfl\log_2(t)\rfl + 1$.
Since at most $O(\log (t))$ intervals contain any given time point
$t$, the time complexity of the meta algorithm is a factor
$O(\log(t))$ larger than that of the black-box $\cB$.

The following Lemma from~\citet{daniely15strongly} shows that an
arbitrary interval $I$ can be partitioned into a sequence of smaller
blocks whose lengths successively double, then successively halve.
This result is key to the usefulness of the geometric covering
intervals.
\begin{lem}\emph{\cite[Lemma~5]{daniely15strongly}} \label{lem:geo-intv}
  Any interval $I \subseteq \mathds{N}$ can be partitioned into two
  finite sequences of disjoint and consecutive intervals, denoted
  $\{J^{(-a)}, J^{(-a+1)}, \ldots, J^{(0)}\}$ and $\{J^{(1)}, J^{(2)},
  \ldots, J^{(b)}\}$ where for all $i \in [(-a)..b]$, we have $J^{(i)}
  \in \cJ$ and $J^{(i)} \subset I$, such that
  \begin{align*}
    |J^{(-i)}| / |J^{(-i+1)}| & \le 1/2, \quad i=1,2,\dotsc,a; \\
    |J^{(i+1)}| / |J^{(i)}| & \le 1/2, \quad i=1,2,\dotsc,b-1~.
  \end{align*}
\end{lem}
\begin{figure}[t]
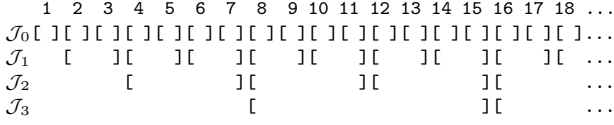

{  \centering
  \begin{minipage}{0.65\textwidth}
{\fontsize{7.3}{9.3}
\begin{Verbatim}[commandchars=\\\{\},codes={\catcode`$=3\catcode`_=8}]
\;   1  2  3  4  5  6  7  8  9 10 11 12 13 14 15 16 17 18 ...
$\cJ_0$[ ][ ][ ][ ][ ][ ][ ][ ][ ][ ][ ][ ][ ][ ][ ][ ][ ][ ]...
$\cJ_1$   [    ][    ][    ][    ][    ][    ][    ][    ][  ...
$\cJ_2$         [          ][          ][          ][        ...
$\cJ_3$                     [                      ][        ...
\end{Verbatim}
}
\end{minipage}
}
\caption{Geometric covering intervals. Each interval is denoted by {\tt [ ]}.}
\label{fig:gc}
\vspace{-20pt}
\end{figure}
%
\paragraph{Regret Decomposition.}
We show now how to use the geometric covering intervals to decompose
the SA-Regret of a complete algorithm $\cM\la\cB\ra$. Using the
notation
\[
  R^\cA_I(\w) := \sum_{t\in I} f_t(\x^\cA_t) - \sum_{t\in I} f_t(\w)\;,
\]
we can restate \eqref{sa-regret} as follows:
\[
  \text{SA-Regret}_T^\cA(I) = \max_{\w\in\cW} R^\cA_I (\w).
\]
Suppose we denote by $\x^{\cB_J}_t$ the decision from black-box run
$\cB_J$ at time $t$ and by $\x^{\cM\la\cB\ra}_t$ the combined decision
of the meta algorithm at time $t$.  Since $\cM\la\cB\ra$ is a
combination of a meta $\cM$ and a black-box $\cB$, its regret depends
on both $\cM$ and $\cB$.  Perhaps surprisingly, we can decompose the
two sources of regret additively through the geometric covering $\cJ$,
as we now describe.  For any $I \subseteq [T]$, let
$\bigcup_{i=-a}^{b} J^{(i)}$ be the partition of $I$ obtained from
Lemma~\ref{lem:geo-intv}.  The regret of $\cM\la\cB\ra$ on $I$ can be
decomposed as follows:
\begin{align}
&R^{\cM\la\cB\ra}_I(\w) \notag\\
&= \sum_{t\in I} \lt(f_t(\x^{\cM\la\cB\ra}_t) - f_t(\w) \rt) \notag \\
&= \sum_{i=-a}^{b} \Bigg(\sum_{t\in J^{(i)}} f_t(\x^{\cM\la\cB\ra}_t) - f_t(\x^{\cB_{J^{(i)}}}_t)
   + f_t(\x^{\cB_{J^{(i)}}}_t) - f_t(\w) \Bigg) \notag\\
&= \underbrace{
      \sum_{i=-a}^{b} \underbrace{ \sum_{t\in J^{(i)}} \lt(f_t(\x^{\cM\la\cB\ra}_t) - f_t(\x^{\cB_{J^{(i)}}}_t) \rt) }_{=:\text{\small  (meta regret on $J^{(i)}$)}}
   }_{=:\text{\small (meta regret on $I$)}}
      + \sum_{i=-a}^{b} \underbrace{ \sum_{t\in J^{(i)}} \lt(f_t(\x^{\cB_{J^{(i)}}}_t) - f_t(\w)\rt) }_{=:\text{\small  (black-box regret on $J^{(i)}$)}} . \label{regret-decomposition}
\end{align}
(We purposely use symbol $J$ for intervals in $\cJ$ and $I$ for a
generic interval that is not necessarily in $\cJ$.)  The black-box
regret on $J=[J_1..J_2]\in \cJ$ is exactly the standard regret for $T
= |J|$, since the black-box run $\cB_J$ was started from time $J_1$.
Thus, in order to prove that a meta algorithm $\cM$ suffers low
SA-Regret, it remains to show two things:
\begin{enumerate}[topsep=0pt,itemsep=0ex,partopsep=1ex,parsep=1ex]
  \item $\cM$ has low regret on interval $J\in\cJ$;
  \item The outer sums over $i$ in~\eqref{regret-decomposition} are
    small for both the meta algorithm and the black-box algorithm.
\end{enumerate}
\citet{daniely15strongly} address the second issue above in their
analysis.  They show that if the black-box regret on $J^{(i)}$ is
$c\sqrt{|J^{(i)}|}$ for some $c$ then the second double summation
of~\eqref{regret-decomposition} is bounded by
$8c\sqrt{|I|}$,\footnote{The argument is essentially the ``doubling
  trick'' described in~\cite[Section~2.3]{cesa-bianchi06prediction}.}
which is perhaps the best one can hope for. The same holds true for
the meta algorithm.  Thus, it remains to focus on the first issue
above. This is our main contribution.

In the next two sections, we describe the design and application of
our meta algorithm.  In Section~\ref{sec:sleeping}, we propose a novel
method that incorporates sleeping experts and the coin betting
framework.  Section~\ref{sec:cbce} describes how our method can be
used as a meta algorithm that has an SA-Regret guarantee.

\section{Coin Betting Meets Sleeping Experts} 
\label{sec:sleeping}

\revision{ 
  Our meta algorithm CBCE extends the coin-betting framework~\cite{orabona16from} to a variant of the learning with expert advice (LEA) problem called ``sleeping experts''~\cite{blum97empirical,freund97using}.
  CBCE is parameter-free (there is no explicit learning rate) and has near-optimal regret.
  Our construction below has further interest as a near-optimal solution for the sleeping bandits problem.
}

%

\paragraph{Sleeping Experts.}
In the LEA framework, the decision set is $\cD = \Delta^N$, an
$N$-dimensional probability simplex of weights assigned to the various
experts. To distinguish LEA from the general online learning problem,
we use notation $\p_t$ in place of $\x_t$, and $h_t$ in place of
$f_t$.  Denoting by $\bfell_t := (\ell_{t,1}, \ldots, \ell_{t,N})^\T
\in [0,1]^{N}$ the vector of loss values of experts at time $t$
provided by the environment, the learner's loss function is $h_t(\p)
:= \p^\T \bfell_t$.


Since $\p \in \cD$ is a probability vector, the learner's decision can
be viewed as hedging between the $N$ alternatives.  Let $\e_i$ be an
indicator vector for dimension $i$; e.g., $\e_2 =
(0,1,0,\ldots,0)^\T$.  In this notation, the comparator set $\cW$ is
$\{\e_1, \ldots, \e_N\}$, that is, the learner competes with a
strategy that commits to a single expert for the entire time interval
$[1..T]$.\footnote{The decision set may be nonconvex, or even
  discrete, for example, $\cD = \{\e_1, \ldots, \e_N\}$
  \cite[Section~3]{cesa-bianchi06prediction}.  In this discrete case,
  no hedging is allowed; the learner must pick a single expert for the
  entire interval.  To choose an element of this set, one could first
  choose an element $\p_t$ from $\Delta^N$, then choose a decision
  $\e_i \in \cD$ with probability $p_{t,i}$.  The regret guarantee for
  such a scheme is the same as for the standard LEA, but with
  \emph{expected} rather than deterministic regret.}

Recall that each black-box run $\cB_J$ is on a different interval $J$.
The meta algorithm's role is to hedge bets over multiple black-box
runs.  Thus, it is natural to treat each run $\cB_J$ as an
\emph{expert} and use an LEA algorithm to combine decisions from each
expert $\cB_J$.  The loss incurred on run $\cB_J$ is $\ell_{t,\cB_J}
:= f_t(\x^{\cB_J}_t)$.

The challenge is that each expert $\cB_J$ may not output decisions at
time steps outside the interval $J$.  This problem can be reduced to
the sleeping experts problem studied
in~\cite{blum97empirical,freund97using}, in which experts are not
required to provide decisions at every time step;
see~\cite{luo15achieving}.  We introduce a indicator variable
$\cI_{t,i} \in \{0,1\}$, which is set to $1$ if expert $i$ is awake
(that is, outputting a decision) at time $t$, and zero otherwise.
Define $\bfcI_t := [\cI_{t,1},\cI_{t,2},\ldots,\cI_{t,N} ]^\T$, where
$N$ can be countably infinite.  The algorithm is said to be ``aware''
of $\bfcI_t$ and it assigns zero weight to the experts that are
sleeping, that is, $\cI_{t,i} = 0 \implies p_{t,i} = 0$.  We would
like to have a guarantee on the regret with respect to  expert $i$, but only for
the time steps for which expert $i$ is awake.
Following~\citet{luo15achieving}, we define a regret bound with respect to  $\mathbf{u}\in\Delta^N$ as follows:
\begin{equation}\label{regret-sleeping}
  \Regret_T(\bfu) := \sum_{t=1}^T\sum_{i=1}^N \cI_{t,i} u_i(\langle \bfell_t,\p_t\rangle - \ell_{t,i}) \;.
\end{equation}
If we set $\bfu=\e_j$ for some $j$, the above is simply regret
with respect to  expert $j$ while that expert is awake, \revision{and we aim to
  achieve a regret of $O(\sqrt{\sum_t \cI_{t,j}})$ up to logarithmic
  factors.}  If $\cI_{t,j} =1 $ for all $t \in [T]$, then it recovers
the standard static regret in LEA.

\paragraph{Coin Betting for LEA.}

We consider the coin betting framework of \citet{orabona16from}, which constructs an LEA algorithm from a \emph{coin betting potential function} (explained below).
A player starts from the initial endowment $1$.
At each time step, the adversary chooses an outcome arbitrarily while the player decides on which side to bet (heads or tails) and how much to bet.
Then the outcome is revealed.
The outcome can be a head (+1), tail (-1), or any point on the continuum between these two extremes (e.g., $-0.3$) where the absolute value of the outcome indicates the weight of being a head or tail.
We encode a coin flip at iteration $t$ as $g_t\in[-1,1]$ where $|g_t|$ indicates the
weight of the outcome.  Let $\Wealth_{t-1}$ be the total money the player possesses
after time step $t-1$. (Note that $\Wealth_0 = 1$.)  We encode the
player's betting decision as the signed betting fraction $\beta_t \in
(-1,1)$, where the positive (negative) sign indicates head (tail) and
the absolute value $|\beta_t| <1$ indicates the fraction of his
current wealth to bet.  Thus, the actual amount wagered is $w_t :=
\beta_t \Wealth_{t-1}$.  Once the coin flip $g_t$ is
revealed, the player's wealth changes as follows: $\Wealth_t = (1+g_t
\beta_t) \Wealth_{t-1}$.  The player makes (loses) money when the
betted side is correct (wrong), and the amount of wealth change
depends on both the flip weight $|g_t|$ and his betting amount
$|w_t|$.

In the coin betting framework, a potential function denoted by
$F_t(g_1,\ldots,g_t)$ has an important role.  Given this function, and
denoting $g_{1:t} := g_1,g_2,\ldots,g_t$, the betting fraction
$\beta_t$ and the amount wagered $w_t$ are determined as follows:
\begin{subequations} \label{eq:s1}
\begin{align}
  \beta_t(g_{1:t-1}) &:= \fr{F_t(g_{1:t-1},1) - F_t(g_{1:t-1},-1)}{ F_t(g_{1:t-1},1) + F_t(g_{1:t-1},-1) }, \label{eq:betting-fraction0} \\
  w_t &= \beta_t(g_{1:t-1}) \cdot \lt( 1 + \sum_{s=1}^{t-1} g_s w_s \rt) \label{eq:betting-amount0} \;.
\end{align}
\end{subequations}
(We use $\beta_t$ in place of $\beta_t(g_{1:t-1})$ when it is clear
from the context.) \revision{A precise definition of $F_t$ appears in
  Section~\ref{sec:potential}; it suffices for now to say that the
  sequence $F_1, F_2, \ldots$ must satisfy the following key
  condition} 
by~\eqref{eq:betting-fraction0}:
\begin{align}\label{eq:wealth-lb} 
  \forall t,\; F_t\left(g_{1:t}\right) \le 1 + \sum_{s=1}^{t} g_s w_{s} \;.
\end{align}
That is, $F_t$ is a lower bound on the wealth of a player who bets by
\eqref{eq:betting-fraction0}.  We emphasize that the term $w_t$ is
decided before $g_t$ is revealed, yet the inequality
\eqref{eq:wealth-lb} holds for any $g_t \in [-1,1]$.
\revision{Property~\eqref{eq:wealth-lb} is key to analyzing the wealth
  arising from the strategy~\eqref{eq:betting-fraction0}; see
  Section~\ref{sec:potential}.  In the restricted setting in which
  $g_s \in \{\pm 1\}$, a betting strategy $\beta_t$ based on a
  potential function proposed by~\citet{krichevsky81theperformance}
  achieves the optimal wealth up to constant
  factors~\cite{cesa-bianchi06prediction}.}

\citet{orabona16from} have devised a reduction of LEA to the simple
coin betting problem described above.  The idea is to instantiate a
coin betting problem for each expert $i$ where the signed coin flip
$g_{t,i}$ is set as a conditionally truncated regret with respect to
expert $i$, rather than being set by an adversary.  We denote by
$\beta_{t,i}$ the betting fraction for expert $i$ and by $w_{t,i}$ the
amount of betting for expert $i$, $\forall i\in[N]$.

We apply this treatment to the sleeping experts setting and propose a
new algorithm \textbf{Sleeping CB}.  Modifications are required
because some experts may not output a decision for some time steps.
Defining $z_{t,i} := \cI_{t,i} g_{t,i}$, we
modify~\eqref{eq:s1} as
follows:
\begin{subequations} \label{eq:s2}
\begin{align}
  \beta_{t,i}(z_{1:t-1,i}) &:= \fr{F_{t,i}(z_{1:t-1,i},1) - F_{t,i}(z_{1:t-1,i},-1)}{ F_{t,i}(z_{1:t-1,i},1) + F_{t,i}(z_{1:t-1,i},-1) },  \label{eq:betting-fraction}\\
  w_{t,i} &= \beta_t(z_{1:t-1,i})\cdot\lt( 1 + \sum_{s=1}^{t-1} z_{s,i} w_{s,i} \rt) \;. \notag
\end{align}
\end{subequations}
Condition~\eqref{eq:wealth-lb} on the potential functions is modified
accordingly to
\begin{equation} 
  \label{eq:wealth-lb-sleeping} 
  \forall t,\; F_{t,i} \left(z_{1:t,i}\right) \le 1 + \sum_{s=1}^{t} z_{s,i} w_{s,i} \;.
\end{equation}
We denote by $\bfpi_{\bfcI_t}$ the prior $\bfpi$ restricted to experts
that are awake (for which $\cI_{t,i}=1$). The Sleeping CB algorithm is
specified in Algorithm~\ref{alg:cblea-sleeping}. (Here and
subsequently, we use notation $[x]_+ := \max(x,0)$.)

\begin{algorithm}[t]
{\small
\begin{algorithmic}
  \STATE \textbf{Input}: Number of experts $N$, prior distribution $\bfpi \in \Delta^{N}$
  \FOR {$t=1,2,\ldots$}
    \STATE For each $i\in \text{Active}(t)$, set \\
    \quad $w_{t,i} \larrow \beta_{t,i}(z_{1:t-1,i}) \cdot \left(1 + \sum_{s=1}^{t-1} z_{s,i} w_{s,i} \right) $. 
    \STATE For all $i\in[N]$, set $\hatp_{t,i} \larrow   \pi_i \cI_{t,i} [w_{t,i}]_+$.
    \STATE Predict with $\p_t \larrow \begin{cases} 
      \hatbfp_t / ||\hatbfp_t||_1 & \text{ if } ||\hatbfp_t||_1 > 0 \\
      \bfpi_{\bfcI_t}             & \text{ if } ||\hatbfp_t||_1 = 0. 
    \end{cases}  $
    \STATE Receive loss vector $\bfell_t \in [0,1]^N$.
    \STATE The learner suffers loss $h_t(\p_t) = \la \bfell_t, \p_t \ra_{\bfcI_t}$.
    \STATE For each $i \in \text{Active}(t)$, set \\
    \qquad$ g_{t,i} \larrow \begin{cases}
          h_t(\p_t) - \ell_{t,i}      & \text{ if } w_{t,i} > 0 \\
          [h_t(\p_t) - \ell_{t,i}]_+  & \text{ if } w_{t,i} \le 0.
        \end{cases}$
  \ENDFOR
\end{algorithmic}
\caption{Sleeping CB} 
\label{alg:cblea-sleeping}
}
\end{algorithm}

The regret of Sleeping CB is bounded in
Theorem~\ref{thm:cblea-sleeping}.  Unlike the standard CB, in which
all the experts use $F_t$ at time $t$, expert $i$ in Sleeping CB uses
$F_{t,i}$, which is different for each expert.  For this reason, the
proof of the CB regret in~\cite{orabona16from} does not transfer
easily to the regret~\eqref{regret-sleeping} of Sleeping CB. However,
this result is crucial to an improved strongly-adaptive regret bound.
\begin{thm} \label{thm:cblea-sleeping}
  \emph{(Regret of Sleeping CB)} 
  Define $S_{t,i} := \sum_{s=1}^{t} \cI_{s,i}$ and for every $i \in [N]$ let $\{F_{t,i}\}_{t\ge1}$ be a sequence of potential functions that satisfies~\eqref{eq:wealth-lb-sleeping}. 
  Suppose that
\[
\log (F_{T,i} (z_{1:T})) \ge H_{T,i}(z_{1:T,i}) := c_1 \fr{(\sum_{s=1}^T z_{s,i})^2}{S_{T,i}} + c_{2,i}, \quad \mbox{for all $i \in [N]$,}
\]
for some $c_1>0$ and $c_{2,i}\in\dsR$.  Then for the regret defined in
\eqref{regret-sleeping}, Algorithm~\ref{alg:cblea-sleeping} satisfies
\[
    \Regret_T(\bfu) \le \sqrt{ \frac{1}{c_1} \cdot \lt( \sum_{i=1}^N u_i S_{T,i} \rt) \cdot \lt( \text{\emph{KL}}(\bfu || \bfpi) - \sum_{i=1}^N u_i c_{2,i} \rt) } \;. 
\]
\end{thm}
\begin{proof}
We show first that $ \sum_{i=1}^N \pi_i \cI_{t,i} g_{t,i} w_{t,i} \le
0$.  
Define $r_{t,i} := \la \bfell_t,\p_t \ra_{\bfcI_t} - \ell_{t,i}$.
Using the fact that 
\[
\sum_{i:\pi_i\cI_{t,i} w_{t,i} > 0}  p_{t,i}r_{t,i} 
= \sum_{i: \cI_{t,i} = 1} p_{t,i}r_{t,i} 
= \sum_{i: \cI_{t,i} = 1} p_{t,i} \la \bfell_t,\p_t \ra_{\bfcI_t} - 
\sum_{i: \cI_{t,i} = 1} p_{t,i} \ell_{t,i} = 0,
\]
we have
  \begin{align*}
    \sum_{i=1}^N   \pi_i \cI_{t,i} g_{t,i} w_{t,i} 
      &= \sum_{i:\pi_i \cI_{t,i} w_{t,i} > 0} \pi_i [w_{t,i}]_+ r_{t,i} 
       + \sum_{i: \pi_i\cI_{t,i} w_{t,i}\le 0} \pi_i \cI_{t,i} w_{t,i}[r_{t,i}]_+ 
    \\&= ||\hatbfp_t||_1 \sum_{i:\pi_i\cI_{t,i} w_{t,i} > 0}  p_{t,i}r_{t,i} 
       + \sum_{i: \pi_i\cI_{t,i} w_{t,i}\le 0}  \pi_i\cI_{t,i} w_{t,i}[r_{t,i}]_+ 
    \\&= 0 + \sum_{i:  \pi_i\cI_{t,i} w_{t,i}\le 0} \pi_i \cI_{t,i} w_{t,i}[r_{t,i}]_+  
    \le 0 \;.
  \end{align*}
  Then, because of the property \eqref{eq:wealth-lb-sleeping} of the
  coin betting potentials, we have
  \begin{align}\label{constantbound}
    \sum_{i=1}^N \pi_i F_{T,i}\lt( z_{1:T,i} \rt) \le 1 + \sum_{i=1}^N \pi_i \sum_{t=1}^T \cI_{t,i} g_{t,i} w_{t,i} \le 1,
  \end{align}
  and
  \begin{align*}
  \sum_{i=1}^N u_i \log (F_{T,i}(z_{1:T,i}))
  &= \sum_{i=1}^N u_i \lt( \log\lt(\fr{u_i}{\pi_i}\rt) + \log\lt(\fr{\pi_i}{u_i} \cdot F_{T,i}(z_{1:T,i}) \rt) \rt) \\
  &\stackrel{\text{(Jensen's)}}{\le} \text{KL}(\bfu||\bfpi) + \log\lt(\sum_{i=1}^N u_i \cdot \fr{\pi_i}{u_i} \cdot F_{T,i}(z_{1:T,i}) \rt)   \\
  &\stackrel{\eqref{constantbound}}{\le} \text{KL}(\bfu||\bfpi) \;.
  \end{align*}
Define $Z_{T,i} := \sum_{t=1}^T z_{t,i}$ and $H'_{T,i}(x) := c_1
\fr{x^2}{S_{T,i}} + c_{2,i}$.  We see from the definition of $H_{T,i}$
in Theorem~\ref{thm:cblea-sleeping} that $H'_{T,i}(Z_{T,i}) =
H_{T,i}(z_{1:T,i})$.  Since $H'$ is symmetric around 0, its inverse
${H'}^{-1}$ usually maps to two distinct values with opposite sign.
To resolve this ambiguity, we define it to map to the nonnegative real
value.  Then, for any comparator $\bfu \in \Delta^N$, we have
  \begin{align*}
    \text{Regret}_T(\bfu) &= \sum_{t=1}^T\sum_{i=1}^N \cI_{t,i} u_i(\langle \bfell_t,\p_t\rangle - \ell_{t,i}) 
     \le \sum_{t=1}^T \sum_{i=1}^N \cI_{t,i} u_i g_{t,i} 
     = \sum_{i=1}^N u_i Z_{T,i} \\
    &= \sum_{i=1}^N u_i {H'}^{-1}_{T,i} ( H'_{T,i} (Z_{T,i}) ) 
     = \sum_{i=1}^N u_i {H'}^{-1}_{T,i} ( H_{T,i} (z_{1:T,i}) ) \\
    &\le \sum_{i=1}^N u_i {H'}^{-1}_{T,i} ( \log (F_{T,i} (z_{1:T,i})) ) \\
    &= \sum_{i=1}^N u_i \sqrt{c_1^{-1} S_{T,i} \cdot\lt(\log (F_{T,i}(z_{1:T,i})) - c_{2,i}\rt) } \\
    &= \sum_{i=1}^N  \sqrt{u_i c_1^{-1} S_{T,i} }\cdot\sqrt{u_i(\log (F_{T,i}(z_{1:T,i})) - c_{2,i}) } \\
    &\stackrel{(a)}{\le} \sqrt{ c_1^{-1} \lt( \sum_{i=1}^N u_i S_{T,i} \rt) \cdot \lt( \sum_{i=1}^N u_i(\log (F_{T,i}(z_{1:T,i})) - c_{2,i}) \rt) }  \\
    &\le \sqrt{ c_1^{-1} \lt( \sum_{i=1}^N u_i S_{T,i} \rt) \cdot \lt( \text{KL}(\bfu || \bfpi) - \sum_{i=1}^N u_i c_{2,i} \rt) },
  \end{align*}
where $(a)$ is due to the Cauchy-Schwartz inequality (noting that the
factors under the square root are all nonnegative since $\log
F_{T,i}(x) \ge H_{T,i}(x)$).
\end{proof}

Note that if $\bfu=\e_j$, then the regret scales with $S_{T,j}$, which
is essentially the number of time steps at which expert $j$ is awake.

While any potential function satisfying the
condition~\eqref{eq:wealth-lb-sleeping} and symmetricity around 0 can
be used, we present two interesting choices: the Krichevsky-Trofimov
potential and the AdaptiveNormal potential.

\paragraph{Krichevsky-Trofimov Potential}

The Krichevsky-Trofimov (KT) potential~\cite{orabona16from} is defined
as follows:
\begin{align} \label{eq:def-kt-potential}
  F_t(g_{1:t}) = \fr{2^t\cdot\Gamma(\dt+1)\cdot\Gamma\lt(\fr{t+\dt+1}{2}+\fr{\sum_{s=1}^t g_s}{2}\rt)\cdot\Gamma\lt(\fr{t+\dt+1}{2}-\fr{\sum_{s=1}^t g_s}{2}\rt)}{\Gamma(\fr{\dt+1}{2})^2\cdot\Gamma(t+\dt+1)}, 
\end{align}
where $\dt \ge 0$ is a time shift parameter set to 0 in this work.
\citet{orabona16from} show that the KT potential
satisfies~\eqref{eq:wealth-lb}.
We  modify the KT potential as follows to handle sleeping
experts by replacing $t$ in several places by $S_{t,i}$:
\begin{align} \label{eq:def-kt-potential-sleeping}
  F_{t,i}(z_{1:t,i}) = \fr{ 2^{S_{t,i}}\cdot\Gamma(\dt+1)\cdot\Gamma\lt(\fr{S_{t,i}+\dt+1}{2}+\fr{\sum_{s=1}^t z_{s,i}}{2}\rt)\cdot\Gamma\lt(\fr{S_{t,i}+\dt+1}{2}-\fr{\sum_{s=1}^t z_{s,i}}{2}\rt) }{ \Gamma(\fr{\dt+1}{2})^2\cdot\Gamma(S_{t,i}+\dt+1) } \;,
\end{align}
which satisfies~\eqref{eq:wealth-lb-sleeping}.\footnote{This is
  trivially implied by the fact that~\eqref{eq:def-kt-potential}
  satisfies~\eqref{eq:wealth-lb} since the only modification
  in~\eqref{eq:def-kt-potential-sleeping} compared
  to~\eqref{eq:def-kt-potential} is to allow ``individual clock''
  $S_{t,i}$ that counts the time steps expert $i$ was awake up to $t$.
} The betting fraction $\beta_t$ defined
in~\eqref{eq:betting-fraction0} with KT potential exhibits the simple
form $\beta_t = \fr{\sum_{s=1}^{t-1} g_{s}}{t+\dt}$ and, for sleeping
experts, we have $\beta_{t,i} = \fr{\sum_{s=1}^{t-1} z_{s,i}}{S_{t,i}
  + \dt} $.  We present the regret of
Algorithm~\ref{alg:cblea-sleeping} with the KT potential in
Corollary~\ref{cor:cblea-sleeping}.
\begin{cor} \label{cor:cblea-sleeping}
  Let $\dt=0$.
  Suppose we run Algorithm~\ref{alg:cblea-sleeping} with the KT potential.
  Then,
  \begin{align*}
    \Regret_T(\bfu) 
    \le \sqrt{2\lt(\sum_{i=1}^N u_i S_{T,i} \rt)\cdot \lt( \text{\emph{KL}}(\bfu || \bfpi) + \fr{1}{2} \ln(T) + 2 \rt) } \;.
  \end{align*}
\end{cor}
\begin{proof}
  Define $H'_{T,i}(x) := \fr{x^2}{2S_{T,i}} + \fr{1}{2}
  \ln(\fr{1}{S_{T,i}}) - \ln(e\sqrt{\pi}) $. Note that this definition
  is identical with the one used in Theorem~\ref{thm:cblea-sleeping}
  if we set $c_1 = \fr{1}{2}$ and
  $c_{2,i}=\fr{1}{2}\ln(\fr{1}{S_{T,i}}) - \ln
  (e\sqrt{\pi})$. According to~\citet[Lemma 15]{orabona16from} with
  $\dt=0$, we have $H'_{T,i}(Z_{T,i}) = H_{T,i}(z_{1:T,i}) \le \ln
  F_{T,i}(z_{1:T,i})$.  It follows that 
  \[
    -\sum_{i=1}^N u_i c_{2,i}
    = \sum_{i=1}^N u_i \lt( (1/2) \ln(S_{T,i}) + \ln(e\sqrt{\pi}) \rt) 
    \le \fr{1}{2} \ln(T) + 2 \;.
  \]
  By plugging $c_1$ and $c_{2,i}$ into
  Theorem~\ref{thm:cblea-sleeping}, we obtain the result.
\end{proof} 

\paragraph{AdaptiveNormal Potential}

Let $\barG_t := \sum_{s=1}^t |g_s|$.  The AdaptiveNormal (AN)
potential proposed by~\citet{orabona17backprop} is defined by
\begin{align} \label{eq:def-exp-potential}
  F_t( g_{1:t}) = \exp\lt( \fr{ (\sum_{s=1}^t g_s)^2}{2(\xi + \barG_t)} - \sum_{s=1}^t \fr{|g_s|}{2(\xi + \barG_{s-1}+1)} \rt)  \;,
\end{align}
where $\xi > 0$ is a parameter of minor importance in our context that
we set to 1.  \citet[Lemma~2]{orabona17backprop} show that the AN
potential satisfies the condition~\eqref{eq:wealth-lb}.  Let
$\barZ_{t,i} := \sum_{s=1}^t |z_{s,i}|$.  For sleeping experts, we use
the following potential that satisfies~\eqref{eq:wealth-lb-sleeping}
due to a trivial consequence of~\eqref{eq:def-exp-potential}
satisfying~\eqref{eq:wealth-lb}:
\begin{align} \label{eq:def-exp-potential-sleeping}
  F_{t,i}( z_{1:t,i}) =  \exp\lt( \fr{ (\sum_{s=1}^t z_{s,i})^2 }{ 2(\xi + \barZ_{t,i}) } - \sum_{s=1}^t \fr{|z_{s,i}|}{2(\xi + \barZ_{s-1,i}+1)} \rt)  \;.
\end{align}
The betting fraction~\eqref{eq:betting-fraction} using the AN
potential can be simplified to 
\[
\beta_t = 2\sig\lt(\fr{ 2
  \sum_{s=1}^{t-1} g_s }{ \xi + \barG_{t-1} + 1}\rt) - 1, \quad \mbox{where}
\; \sig(z) = \fr{1}{1+\exp(-z)}.
\]
For sleeping experts, we have
\[
\beta_{t,i} = 2\sig\lt(\fr{ 2 \sum_{s=1}^{t-1}
  z_{s,i} }{ \xi + \barZ_{t-1,i} + 1}\rt) - 1.
\]

Define $r_{t,i} := \cI_{t,i}(h_t(\p_t) - \ell_{t,i})$ and $\tilL_{T,i}
:= \sum_{t=1}^T [ -r_{t,i} ]_+$.  We present the regret bound of
Sleeping CB with the AN potential in
Corollary~\ref{cor:cblea-sleeping-exp}.
\begin{cor}\label{cor:cblea-sleeping-exp}
  Let $\xi = 1$.
  Suppose we run Algorithm~\ref{alg:cblea-sleeping} with the AN
  potential.  
  Let $W_\bfu := 1+\sum_{i=1}^N u_i \barZ_{T,i}$ and $L_{T,i} := \sum_{t=1}^T \cI_{t,i}\ell_{t,i}$.  Then we have
  \[
      \textstyle    \Regret_T(\bfu) = \sqrt{2 W_\bfu \lt(\KL(\bfu||\bfpi) + \fr{1}{2} \ln(W_\bfu) \rt)} \;,
  \]
and moreover
  \[
      \textstyle \Regret_T(\bfu) =  O\left( \sqrt{(\sum_{i=1}^N u_i L_{T,i})\cdot(\KL(\bfu||\bfpi) + \ln(W_\bfu))} + \KL(\bfu||\bfpi) + \ln(W_\bfu) \right).
  \]
\end{cor}
\begin{proof}
Define 
\[
H'_{T,i}(x) := \fr{x^2}{2(1 + \barZ_{T,i})} - \sum_{s=1}^T
\fr{|z_{s,i}|}{2(1 + \barZ_{s-1,i} + 1)}.
\]  
Note that $H'_{T,i}(Z_{T,i}) = H_{T,i}(z_{1:T,i}) =
\ln(F_{T,i}(z_{1:T,i}))$.  To match this definition with the setup of
Theorem~\ref{thm:cblea-sleeping}, we redefine $S_{T,i} := 1 +
\barZ_{T,i}$, \revision{for which the theorem still holds, since
  $S_{T,i}$ is used only through the function $H_{T,i}$.}  We also set
\[
c_1 = \frac{1}{2}, \quad
c_{2,i} = -\sum_{s=1}^T \fr{|z_{s,i}|}{2(1 +
  \barZ_{s-1,i} + 1)}.
\]
 Using $|z_{s,i}| \le 1$ and $ \fr{a-b}{a} \le \ln(a) - \ln(b) $, we have
\begin{align*}
    - c_{2,i} 
    &\le \fr{1}{2} \sum_{s=1}^T \fr{|z_{s,i}|}{(1 + \barZ_{s-1,i} + |z_{s,i}|)} 
  \\&= \fr{1}{2} \sum_{s=1}^T \left(\ln(1 + \barZ_{s,i}) - \ln(1 + \barZ_{s-1,i})\right)
     = \fr{1}{2}\ln(1 + \barZ_{T,i}),
\end{align*}
so it follows from Jensen's inequality that 
\[
\sum_{i=1}^N - c_{2,i}
    \le \fr{1}{2} \ln\lt( 1 + \sum_{i=1}^N u_i \barZ_{T,i} \rt).
\]
Then, by Theorem~\ref{thm:cblea-sleeping}, we have 
\[
    \textstyle \Regret_T(\bfu) = \sqrt{2\lt( 1 + \sum_{i=1}^N u_i \barZ_{T,i} \rt)\cdot\lt(\KL(\bfu||\bfpi) + \fr{1}{2} \log\lt(1+\sum_{i=1}^N u_i \barZ_{T,i}\rt) \rt)} \;,
\]
proving the first statement of the theorem.

For the second statement, we use Lemma~\ref{lem:barZ_to_tilL} in
Section~\ref{sec:technical}, which says if $\Regret_T(\bfu) \le
\sqrt{( 1 + \sum_{i=1}^N u_i \barZ_{T,i}) A(\bfu)}$ for some function
$A(\bfu)$ then
  \[
\textstyle\Regret_T(\bfu) = \sqrt{(1 + 2\sum_{i=1}^N u_i \tilL_{T,i})
  A(\bfu)} + A(\bfu)\;.
\]
  Setting $A(\bfu) = \KL(\bfu||\bfpi) + \fr{1}{2} \log \left(1+\sum_{i=1}^N
  u_i \barZ_{T,i} \right)$ and the fact that $[-r_{t,i}]_+ \le \ell_{t,i}
  \implies \tilL_{T,i} \le L_{T,i}$ we verify the second claim.
\end{proof}

When we set $\bfu = \e_{i}$ with this AN potential, we obtain a regret
bound that scales with $L_{T,i}$, which is always smaller than
$S_{T,i}$. The difference becomes significant when the expert $i$
suffers a loss $\ell_{t,i}$ that is close to 0 for all $t \in
\Active(i)$.  Note that Sleeping CB with the AN potential is quite
similar to AdaNormalHedge~\cite{luo15achieving}, which has the same
regret order. The key difference is that the truncation operates in
the potential function for AdaNormalHedge whereas for ours it operates
in the reduction to LEA (see the definition of $g_{t,\cB_J}$).

According to our results, the regret bound of the KT potential can be
much larger than that of the AN potential.  Thus, one might wonder if
we should always use the AN potential.  Our empirical study in
Section~\ref{sec:expr} shows a case where KT has a benefit over AN.

\section{Coping with a Changing Environment by Sleeping CB}
\label{sec:cbce}
In this section, we synthesize the results in Sections~\ref{sec:meta}
and~\ref{sec:sleeping} to specify and analyze our meta algorithm.
Recall that a meta algorithm must efficiently aggregate decisions from
multiple black-box runs that are active at time $t$.  By treating each
black-box run as an expert, we use Sleeping CB
(Algorithm~\ref{alg:cblea-sleeping}) as the meta algorithm, with
geometric covering intervals.
\revision{An important motivation for the use of Sleeping CB is that
  it is parameter-free. Other sleeping bandits techniques require the
  number of black-box runs (experts) to be specified in advance, which
  results in a theoretical guarantee only up to some finite time
  horizon $T$. By contrast, our approach provides an ``anytime''
  guarantee.}
The complete algorithm, which we call \textbf{Coin Betting for
  Changing Environments (CBCE)}, is shown in Algorithm~\ref{alg:cbce}.

\begin{algorithm}[t]
{\small
\begin{algorithmic}
  \STATE \textbf{Input}: A black-box algorithm $\cB$ and a prior distribution $\bfpi \in \Delta^{|\cJ|}$ over $\{\cB_J \mid J \in \cJ\}$. 
  \FOR {$t = 1$ \TO $T$}
  \STATE For each $J\in\Active(t)$, set \\
  \quad $w_{t,\cB_J} \larrow \beta_{t,\cB_J}(z_{t,\cB_J})\cdot(1 + \sum_{s=1}^{t-1} z_{s,\cB_J} w_{s,\cB_J}) $
  \STATE Set $\hatp_{t,\cB_J} \larrow \pi_{\cB_J} \cI_{t,\cB_J}[w_{t,\cB_J}]_+$ for $J\in \text{Active}(t)$ and 0 for $J\not\in\Active(t)$.
  \STATE Compute $\p_t \larrow \begin{cases} 
    \hatbfp_t / ||\hatbfp_t||_1 & \text{ if } ||\hatbfp_t||_1 > 0 \\
    [\pi_{\cB_J}]_{J \in \Active(t)}  & \text{ if } ||\hatbfp_t||_1 = 0. 
  \end{cases}  $
  \STATE The black-box run $\cB_J$ picks a decision $\x^{\cB_J}_t \in \cD$, $\forall J \in \text{Active}(t)$.
  \STATE The learner picks a decision $\x_t = \sum_{J \in \cJ} p_{t,\cB_J} \x^{\cB_J}_t $.
  \STATE Each black-box run $\cB_J$ that is awake ($J\in\text{Active}(t)$) suffers loss $\ell_{t,\cB_J} := f_t(\x^{\cB_J}_t)$.
  \STATE The learner suffers loss $f_t(\x_t)$. 
  \STATE For each $J \in \text{Active}(t)$, set \\
  \qquad $g_{t,\cB_J} \larrow \begin{cases}
    f_t(\x_t) - \ell_{t,\cB_J}  & \text{ if } w_{t,\cB_J} > 0 \\ 
    [f_t(\x_t) - \ell_{t,\cB_J}]_+ & \text{ if } w_{t,\cB_J} \le 0.
  \end{cases}$
  \ENDFOR
\end{algorithmic}
\caption{Coin Betting for Changing Environment (CBCE)}
\label{alg:cbce}
}
\end{algorithm}
We first present the results with the KT potential and then discuss
applying the AN potential in the same manner.  We make use of the
following assumption.
\begin{assump}
  \label{ass:convexity}
  The loss function $f_t$ is convex and maps to $[0,1]$,  $\forall t\in\dsN$.
\end{assump}
Nonconvex loss functions can be accommodated by randomized decisions:
We choose the decision $\x_t^{\cB_J}$ from black-box $\cB_J$ with
probability $p_{t,\cB_J}$.  It is not difficult to show that the same
regret bound holds, but now in expectation.  When loss functions are
unbounded, they can be scaled and truncated to $[0,1]$. Any
nonconvexity that results can be handled in the manner just described.

We define our choice of prior $\bar\bfpi \in \Delta^{|\cJ|}$ as
follows:
\begin{align}\label{def:barpi}
  \bar \pi_{\cB_J} := Z'^{-1} \lt( \fr{1}{J_1^2 (1 + \lfl\log_2 J_1\rfl)} \rt), \;\; \forall J \in \cJ \;,
\end{align}
where $Z'$ is a normalization factor. Since there exist at most $1 +
\lfl \log_2 J_1 \rfl$ distinct intervals starting at time $J_1$, we
have that $Z'< \sum_{t=1}^\infty t^{-2} = \pi^2/6$.

We bound the meta regret with respect to a black-box run $\cB_J$  as follows.
\begin{lem} \label{lem:cbce}
  \emph{(Meta regret of CBCE with the KT potential)}
  Assume~\ref{ass:convexity}.
Suppose we run CBCE (Algorithm~\ref{alg:cbce}) with a black-box
algorithm $\cB$, prior $\bar\bfpi$, and the KT
potential~\eqref{eq:def-kt-potential-sleeping} with $\dt = 0$.  The
meta regret of CBCE$\la\cB\ra$ on interval $J=[J_1..J_2]\in\cJ$ is
  \begin{align*}
    \sum_{t\in J} f_t(\x^{\emph{\CBCE}\la\cB\ra}_t) - f_t(\x^{\cB_J}_t)
    \le \sqrt{|J| \lt( 7\ln(J_2) + 5\rt) }
    = O(\sqrt{|J|\log J_2}) \;.
  \end{align*}
\end{lem}
\begin{proof}
  Note that our regret definition for meta algorithms 
  \begin{align}\label{pf-lem-cbce-1}
  \sum_{t\in J} f_t(\x_t^{\CBCE\la\cB\ra}) - f_t(\x_t^{\cB_J})\;,
  \end{align}
  is slightly different from that of Theorem~\ref{thm:cblea-sleeping} for $\bfu=\e_i$: $\sum_{t\in J:\cI_{t,i} = 1} \la\bfl_t,\p_t\ra - \ell_{t,i}$.
  This translates to, in the language of meta algorithms, $\sum_{t\in J:\cI_{t,\cB_J} = 1} \la\bfl_t,\p_t\ra_{\bfcI_t} - \ell_{t,\cB_J}$ for $\bfu=\e_{\cB_J}$ (recall $\ell_{t,\cB_J} = f_t(\x_t^{\cB_J})$).

  We claim that Theorem~\ref{thm:cblea-sleeping} and Corollary~\ref{cor:cblea-sleeping} for hold true for the regret~\eqref{pf-lem-cbce-1}.
  Note that, using Jensen's inequality, $f_t(\x^{\CBCE\la\cB\ra}_t) \le \la \bfell_t,\p_t \ra_{\bfcI_t}$.
  Then, in the proof of Theorem~\ref{thm:cblea-sleeping}
  \begin{align*}    
  &\sum_{J\in\cJ} \pi_{\cB_J} \cI_{t,\cB_J}g_{t,\cB_J} w_{t,\cB_J}  \\
  &= \sum_{J\in\cJ: \pi_{\cB_J}\cI_{t,\cB_J} w_{t,\cB_J} > 0} \pi_{\cB_J} [w_{t,\cB_J}]_+ (f_t(\x^{\CBCE\la\cB\ra}_t) - \ell_{t,\cB_J})  \\
  &\qquad +  \sum_{J\in\cJ: \pi_{\cB_J}\cI_{t,\cB_J} w_{t,\cB_J}\le 0} \pi_{\cB_J} \cI_{t,\cB_J} w_{t,\cB_J}[f_t(\x^{\CBCE\la\cB\ra}_t) - \ell_{t,\cB_J}]_+ \\
  &\le \sum_{J\in\cJ: \pi_{\cB_J}\cI_{t,\cB_J} w_{t,\cB_J} > 0} \pi_{\cB_J} [w_{t,\cB_J}]_+ (\la \bfell_t,\p_t \ra_{\bfcI_t} - \ell_{t,\cB_J}) \\
  &\qquad +  \sum_{J\in\cJ: \pi_{\cB_J}\cI_{t,\cB_J} w_{t,\cB_J}\le 0} \pi_{\cB_J} \cI_{t,\cB_J} w_{t,\cB_J}[\la \bfell_t,\p_t \ra_{\bfcI_t} - \ell_{t,\cB_J}]_+ \;.
  \end{align*}
  Then, one can see that the proof of Theorem~\ref{thm:cblea-sleeping} goes through, so does Corollary~\ref{cor:cblea-sleeping}.
  
  Since $\KL(\e_{\cB_J} || \bar\bfpi) = \ln\fr{1}{\bar\pi_{\cB_J}} \le \ln\lt(\fr{\pi^2}{6}J_1^{2}(1+\lfl\log_2 J_1\rfl)\rt) \le 3 \ln(J_2) + \fr{1}{2}$, it follows that
  \begin{align*}
  \sum_{t\in J} f_t(\x^{\emph{\CBCE}\la\cB\ra}_t) - f_t(\x^{\cB_J}_t)
  &\stackrel{\text{(Cor.~\ref{cor:cblea-sleeping})}}{\le} \sqrt{2 S_{T,\cB_J} \cdot \lt( \text{{KL}}(\e_{\cB_J} || \bfpi) + \fr{1}{2} \ln(J_2) + 2 \rt) }  \\
  &\le \sqrt{2 |J| \lt( \fr{7}{2} \ln(J_2) + \fr{5}{2} \rt) }  \\
  &= \sqrt{|J| \lt( 7\ln(J_2) + 5\rt) }  \; .
  \end{align*}
\end{proof}

We now present the bound on the SA-Regret $R^{\text{\CBCE}\la\cB\ra}_I (\w)$ with respect to $\w \in\cW$ on intervals $I\subseteq [T]$ that are not necessarily in $\cJ$.
\begin{thm}\label{thm:untitled}
  \emph{(SA-Regret of \texorpdfstring{$\CBCE\la\cB\ra$}{} with the KT potential)}
  Take the same assumption as Lemma~\ref{lem:cbce}.
  Suppose that the black-box algorithm $\cB$ has regret $R^{\cB}_T$ bounded by $A_1 T^\alpha$, where $\alpha \in (0,1)$.
  Let $I = [I_1.. I_2]$.
  The SA-Regret of $\CBCE$ with black-box $\cB$ on the interval $I$ with respect to any $\w \in \cW$ is bounded as follows:
\begin{align*}
  R^{\emph{\CBCE}\la\cB\ra}_I (\w) 
  \le \fr{4}{2^\alpha - 1} A_1|I|^\alpha + 8\sqrt{|I| (7\ln ( I_2) + 5)}
  = O ( A_1 |I|^\alpha + \sqrt{|I| \ln I_2} ) \;.
\end{align*}
\end{thm}
\begin{proof}

By Lemma~\ref{lem:geo-intv}, we know that $J$ can be decomposed into two sequences of intervals $\{J^{(-a)}, \ldots, J^{(0)}\}$ and $\{J^{(1)}, J^{(2)}, \ldots, J^{(b)}\}$.
Continuing from~\eqref{regret-decomposition},
\begin{align*}
  &R^{\CBCE\la\cB\ra}_I(\w)  
\\&= \underbrace{ \sum_{i=-a}^{b} \sum_{t\in J^{(i)}} \lt(f_t(\x^{\CBCE\la\cB\ra}_t) - f_t(\x^{\cB_{J^{(i)}}}_t) \rt) }_{=:E_1}  + \underbrace{ \sum_{i=-a}^{b} \sum_{t\in J^{(i)}}  \lt(f_t(\x^{\cB_{J^{(i)}}}_t) - f_t(\w)\rt) }_{=:E_2} \;. \notag
\end{align*}
Then,
\begin{align*}
  E_1 &= \sum_{i\in [(-a)..0]} \sum_{t \in J^{(i)}} \lt(f_t(\x^{\CBCE\la\cB\ra}_t) - f_t(\x^{\cB_{J^{(i)}}}_t) \rt)  
  \\&\qquad + \sum_{i\in [1..b]}\sum_{t \in J^{(i)}} \lt(f_t(\x^{\CBCE\la\cB\ra}_t) - f_t(\x^{\cB_{J^{(i)}}}_t) \rt) \;.
\end{align*}
The first summation is upper-bounded by, due to Lemma~\ref{lem:cbce} and Lemma~\ref{lem:geo-intv}, $\sum_{i\in [(-a)..0]} \sqrt{|J^{(i)}|(7\ln(I_2) + 5)} \le \sqrt{7\ln(I_2) + 5} \cdot \sum_{i=0}^\infty (2^{-i}|I|)^{1/2} \le \sqrt{7\ln(I_2) + 5} \cdot  (4 \sqrt{|I|}) $.
The second summation is bounded by the same quantity due to symmetry.
Thus, $E_1 \le 8\sqrt{|I| (7\ln( I_2) + 5)}$ .

In the same manner, one can show that $E_2 \le 2 \cdot \fr{2^\alpha}{2^\alpha - 1}|I|^\alpha \le \fr{4}{2^\alpha - 1}|I|^\alpha $, which concludes the proof.
\end{proof}
For the standard LEA problem, one can run the algorithm CB with KT potential (equivalent to Sleeping CB with $\cI_{t,i}=1, \forall t,i$), which achieves static regret $O(\sqrt{ T \log (NT) })$~\cite{orabona16from}.
Using CB as the black-box algorithm, the regret of $\CBCE\la\cB\ra$ on $I$ is $R^{\CBCE\la\CB\ra}_I(\w) \allowbreak= O(\sqrt{ |I| \log (N I_2) })$, and so SA-Regret$^{\CBCE\la\CB\ra}_T(|I|) =O(\sqrt{ |I| \log (N T) }) $.
It follows that the $m$-shift regret of $\CBCE\la\CB\ra$ is $O(\sqrt{mT\log(NT)})$ using the technique presented in Section~\ref{sec:saregret}. 

As said above, our bound improves over the best known result with the same time complexity in~\cite{daniely15strongly}.
The key ingredient that allows us to get a better bound is the Sleeping CB Algorithm~\ref{alg:cblea-sleeping}, that achieves a better SA-Regret than the one of \cite{daniely15strongly}. 
In the next section, we will show that the empirical results also confirm the theoretical gap of these two algorithms.

\paragraph{The AdaptiveNormal Potential}

We present the meta regret bound of CBCE with the AN potential on intervals $J\in\cJ$ in Lemma~\ref{lem:cbce-exp} and on any interval $I$ in Lemma~\ref{thm:untitled-exp}.
\begin{lem} \label{lem:cbce-exp}
  \emph{(Meta regret of CBCE with the AN potential)}
  Assume~\ref{ass:convexity}.
  Suppose we run CBCE (Algorithm~\ref{alg:cbce}) with a black-box algorithm $\cB$, prior $\bar\bfpi$, and the AN potential~\eqref{eq:def-exp-potential-sleeping} with $\xi = 1$.
  The meta regret of CBCE$\la\cB\ra$ on interval $J=[J_1..J_2]\in\cJ$ is
  \begin{align*}
    \sum_{t\in J} f_t(\x^{\emph{\CBCE}\la\cB\ra}_t) - f_t(\x^{\cB_J}_t)
    = O\lt(\sqrt{\sum_{t\in J} f_t(\x_t^{\cB_{J}}) \log(J_2)} + \log (J_2)\rt) \;.
  \end{align*}
\end{lem}
\begin{proof}
  The proof deviates from the proof of Lemma~\ref{lem:cbce}.
  Since $W_{\e_{\cB_J}} = O(J_2)$.
  \begin{align*}
      \sum_{t\in J} f_t(\x^{\CBCE\la\cB\ra}_t) - f_t(\x^{\cB_J}_t)
    &\stackrel{\text{(Cor.~\ref{cor:cblea-sleeping-exp})}}{=} O\lt( \sqrt{  L_{T,\cB_J} \log(J_2 W_{\e_{\cB_J}})  }  + \log(J_2 W_{\e_{\cB_J}}) \rt)
    \\&= O\lt(\sqrt{L_{T,\cB_J}\log (J_2)} + \log (J_2)\rt) \;.
  \end{align*}
\end{proof}
\begin{thm}\label{thm:untitled-exp}
  \emph{(SA-Regret of \texorpdfstring{$\CBCE\la\cB\ra$}{} with the AN potential)}
  Make the same same assumption as Lemma~\ref{lem:cbce}.
  Suppose that the black-box algorithm $\cB$ has regret $R^{\cB}_T(\w)$ bounded by $A_1(\sum_{t=1}^T q_t)^\alpha + A_2$ for some $\{q_t\}$ where $\alpha \in (0,1)$ where $A_2$ grows poly-logarithmically in $T$.
  For any $I = [I_1.. I_2]$ and $\w\in\cW$,
  $R^{\CBCE\la\cB\ra}_I(\w) = O\lt(\min\lt\{A_1|I|^{\alpha} + \sqrt{|I|\log I_2}, U(\w) \rt\}\rt)$ where $U(\w)=$
  \begin{align*}
    A_1 (\log |I|)^{1-\alpha }\lt(\sum_{t\in I} q_t\rt)^\alpha + \log(I_2) \sqrt{\sum_{t\in I} f_t(\w)}  \;,
  \end{align*}
  and we ignore additive terms scaling at most poly-logarithmically in $I_2$.
\end{thm}
\begin{proof}
  The first equation in the minimum operator of the stated regret bound is trivial as $\sum_{t\in J} f_t(\x_t^{\cB_J}) \le |J|$.
  Thus, we focus on $U(\w)$.
  By Lemma~\ref{lem:geo-intv}, we know that $J$ can be decomposed into two sequences of intervals $\{J^{(-a)}, \ldots, J^{(0)}\}$ and $\{J^{(1)}, J^{(2)}, \ldots, J^{(b)}\}$.
  Continuing from~\eqref{regret-decomposition},
  \begin{align*}
    &R^{\CBCE\la\cB\ra}_I(\w)  
  \\&= \underbrace{ \sum_{i=-a}^{b} \sum_{t\in J^{(i)}} \lt(f_t(\x^{\CBCE\la\cB\ra}_t) - f_t(\x^{\cB_{J^{(i)}}}_t) \rt) }_{=:E_1}  + \underbrace{ \sum_{i=-a}^{b} \sum_{t\in J^{(i)}}  \lt(f_t(\x^{\cB_{J^{(i)}}}_t) - f_t(\w)\rt) }_{=:E_2} \;. \notag
  \end{align*}
  Define $L^{\cB}_{J^{(i)}} := \sum_{t\in J^{(i)}} f_t(\x_t^{\cB_{J^{(i)}}})$.
  For simplicity, we denote by $C_{1,i} \sqrt{L^\cB_{J^{(i)}}} + C_{2,i} $ the meta regret bound of CBCE for the interval $J^{(i)}$ (see Lemma~\ref{lem:cbce-exp}).
  Define $\bar C_1 = \max_i C_{1,i}$. Then, due to Lemma~\ref{lem:cbce-exp} and the fact $\sum_{i=1}^k x_i^{\alpha} \le k^{1-\alpha}(\sum_{i=1}^k x_i)^\alpha$ for $\alpha \in (0,1)$, 
  \begin{align*}
    E_1 
      &\le \textstyle \sum_i C_{1,i} \sqrt{L^\cB_{J^{(i)}}} + C_{2,i}
    \\&= \textstyle O( \bar C_{1} \sqrt{\log|I|} \sqrt{ \sum_i L^\cB_{J^{(i)}} } + \sum_i C_{2,i} )
    \\&= \textstyle O( \bar C_{1} \sqrt{\log|I|} \sqrt{ \sum_{t\in I} f_t(\w) +  \sum_i \sum_{t\in J^{(i)}} f_t(\x_t^{\cB_{J^{(i)}}}) - f_t(\w)       } + \sum_i C_{2,i} )
    \\&= \textstyle O( \bar C_{1} \sqrt{\log|I|} \sqrt{ \sum_{t\in I} f_t(\w) + E_2   } + \sum_i C_{2,i} ) \;.
  \end{align*}
  Note that
  \begin{align*}
    E_2 = \textstyle O(  A_1 (\log|I|)^{1-\alpha} (\sum_{t\in I} q_t)^\alpha + A_2\log|I|) \;. 
  \end{align*}
  Note that $\bar C_{1,i} = O( \sqrt{\log I_2} )$ and $ C_{2,i} = O(\log I_2) $.
  For brevity, we ignore the term $\sqrt{E_2}$ that is smaller than $E_2$ unless $E_2 < 1$.
  Ignoring terms that cannot grow faster than poly-logarithmically in $I_2$, the regret of $\CBCE\la\cB\ra$ for interval $I$ can be simplified to
  \[
    \textstyle  O(  A_1 (\log |I|)^{1-\alpha }(\sum_{t\in I} q_t)^\alpha + \log(I_2) \sqrt{\sum_{t\in I} f_t(\w)} ) \;.
  \]
\end{proof}
We emphasize that the order of regret stated in Theorem~\ref{thm:untitled-exp} is always no larger than CBCE with the KT potential.
Furthermore, the regret bound of the AN potential scales roughly with $\min_{w\in\cW} \sum_{t\in I} f_t(\w) + (\sum_{t\in I} q_t)^\alpha$.
In some cases, this form of regret can be much smaller when there exists a decision $\w$ that has very small loss in the interval $I$.
We instantiate the result above for LEA in Corollary~\ref{cor:data-dependent-lea} and for online convex optimization (OCO) in Corollary~\ref{cor:data-dependent-oco}.

\revision{
\begin{cor}\label{cor:data-dependent-lea}
  \emph{(SA-Regret of CBCE$\la$CB$\ra$ with the AN potential)}
  Suppose we run CBCE with the AN potential equipped with CB with the AN potential for LEA.
  Then, for any $I=[I_1..I_2]$, 
  \[
    \SARegret_T^{\CBCE\la\CB\ra}(I) = O\lt( \log(I_2) \sqrt{\lt(\min_{i\in[N]} \sum_{t\in I} \ell_{t,i}\rt)\log N} + \polylog(I_2) \rt) \;.
  \]
\end{cor}
\begin{proof}
  Consider the standard LEA problem where all experts are awake all the time.
  Verify that the regret of CB with the AN potential with respect to expert $i$ is $O(\sqrt{L_{T,i} \log(NT)})$ by Corollary~\ref{cor:cblea-sleeping-exp}, ignoring additive terms that are $O(\text{polylog}(T))$.
  Plugging in $A_1 = \sqrt{\log(N |I|)}$ and $\alpha = 1/2$ in Theorem~\ref{thm:untitled-exp} concludes the proof.
\end{proof}

For LEA, AdaNormalHedge.TV of~\citet{luo15achieving} achieves the SA-Regret bound 
  \[O\textstyle\lt( \sqrt{\lt(\min_{i\in[N]} \sum_{t\in I} \ell_{t,i}\rt)\log (N I_1)} + \polylog(I_2) \rt),\] 
which is at least $\sqrt{\log(I_2)}$ factor smaller than $\CBCE\la\CB\ra$.
However, the time complexity of AdaNormalHedge.TV is $O(NT)$ per time step rather than $O(N\log(T))$.

\begin{cor}\label{cor:data-dependent-oco}
  \emph{(SA-Regret of CBCE$\la$FTRL$\ra$ with the AN potential)}
  Consider the OCO problem where the functions $\{f_t\}$ are nonnegative, smooth (gradients are Lipschitz continuous), and $L$-Lipschitz.
  Suppose we run CBCE with the AN potential and use Follow The Regularized Leader (FTRL) on the linearized losses as the blackbox~\cite{Shalev-Shwartz07,OrabonaCCB15}, setting the regularizer at time $t$ to $\frac{\sqrt{L^2+\sum_{s=1}^{t-1} \|\nabla f_s(\x_s)\|_2^2}}{2}\|\cdot\|_2^2$.
  Then, for any $I=[I_1..I_2]$,
  \[
    \SARegret_T^{\CBCE\la\OGD\ra} (I) = O\lt( \log(I_2) \sqrt{\min_{\w\in\cW} \sum_{t\in I} f_{t}(\w)} + \polylog(I_2)\rt)  \;.
  \]
\end{cor}
\begin{proof}
  One can show that the regret bound of FTRL with the assumed regularizer achieves the regret bound of order $O(\sqrt{L^2+\sum_{t=1}^T f_t(\w)})$ with respect to any $\w$, see, e.g., \cite[Theorem 2]{Shalev-Shwartz07}.
  Then, plugging in $\alpha=1/2$ in Theorem~\ref{thm:untitled-exp} concludes the proof.
\end{proof}
To the best of our knowledge, Corollary~\ref{cor:data-dependent-oco} is the first first-order SA-Regret bound for OCO.
}

\paragraph{Discussion.}
Note that one can obtain the same result using the data streaming intervals (DS)~\cite{hazan07adaptive,gyorgy12efficient} in place of the geometric covering intervals (GC).
Section~\ref{sec:ds} elaborates on this with a lemma stating that DS induces a partition of an interval $I$ in a very similar way to GC (a sequence of intervals of doubling lengths).

Our improved bound has another interesting implication.
In designing strongly-adaptive algorithms for LEA, there is a well known technique called ``restarts'' or ``sleeping experts'' that has time complexity $O(NT^2)$~\cite{hazan07adaptive,luo15achieving}, and several studies used DS or GC to reduce the time complexity to $O(NT\log T)$~\cite{hazan07adaptive,gyorgy12efficient,daniely15strongly}.
However, it was unclear whether it is possible to achieve both an $m$-shift regret of $O(\sqrt{mT(\log N + \log T)})$ and a time complexity of $O(NT\log T)$ without knowing $m$.
Indeed, every study on $m$-shift regret with time $O(NT\log T)$ results in suboptimal $m$-shift regret bounds~\cite{daniely15strongly,gyorgy12efficient,hazan07adaptive}, to our knowledge.
Furthermore, some studies (e.g.,~\cite[Section~5]{luo15achieving}) speculated that perhaps applying the data streaming technique would increase its SA-Regret by a logarithmic factor.
Our analysis implies that one can reduce the overall time complexity to $O(NT\log T)$ without sacrificing the order of SA-Regret and $m$-shift regret.\footnote{Note that we still pay an extra logarithmic factor when it comes to the \emph{first-order} regret bound as discussed right below Corollary~\ref{cor:data-dependent-lea}.}

\section{Experiments}
\label{sec:expr}

We now turn to an empirical evaluation of algorithms for changing environments.
We compare the performance of the meta algorithms under two online learning problems: $(i)$ learning with expert advice (LEA) and $(ii)$ metric learning (ML).
We compare CBCE with SAOL~\cite{daniely15strongly} and AdaNormalHedge.TV (ATV)~\cite{luo15achieving}.
Although ATV was originally designed for LEA only, it is not hard to extend it to a meta algorithm and show that it has the same order of SA-Regret as CBCE using the same techniques.
We run CBCE with both KT and AN potential, which are denoted by CBCE(KT) and CBCE(AN) respectively.

For our empirical study, we replace the geometric covering intervals (GC) with the data streaming intervals (DS)~\cite{hazan07adaptive,gyorgy12efficient}.
Let $u(t)$ be a number such that $2^{u(t)}$ is the largest power of 2 that divides $t$; e.g., $u(12)=2$.
The data streaming intervals are $\cJ = \{[t..(t+g\cdot2^{u(t)}-1)]: t =1,2,\ldots \}$ for some $g \ge 1$.
DS is an attractive alternative, unlike GC, $(i)$ DS initiates one and only one black-box run at each time, and $(ii)$ it is more flexible in that the parameter $g$ can be increased to enjoy smaller regret in practice while increasing the time complexity by a constant factor.

For both ATV and CBCE, we set the prior $\bfpi$ over the black-box runs as the uniform distribution. 
Note that this does not break the theoretical guarantees since the number of black-box runs are never actually infinite; we used $\bar\bfpi$~\eqref{def:barpi} for ease of exposition.

\begin{figure*}
  {\centering
  \begin{tabular}{c}
    \includegraphics[width=.9\textwidth,valign=t]{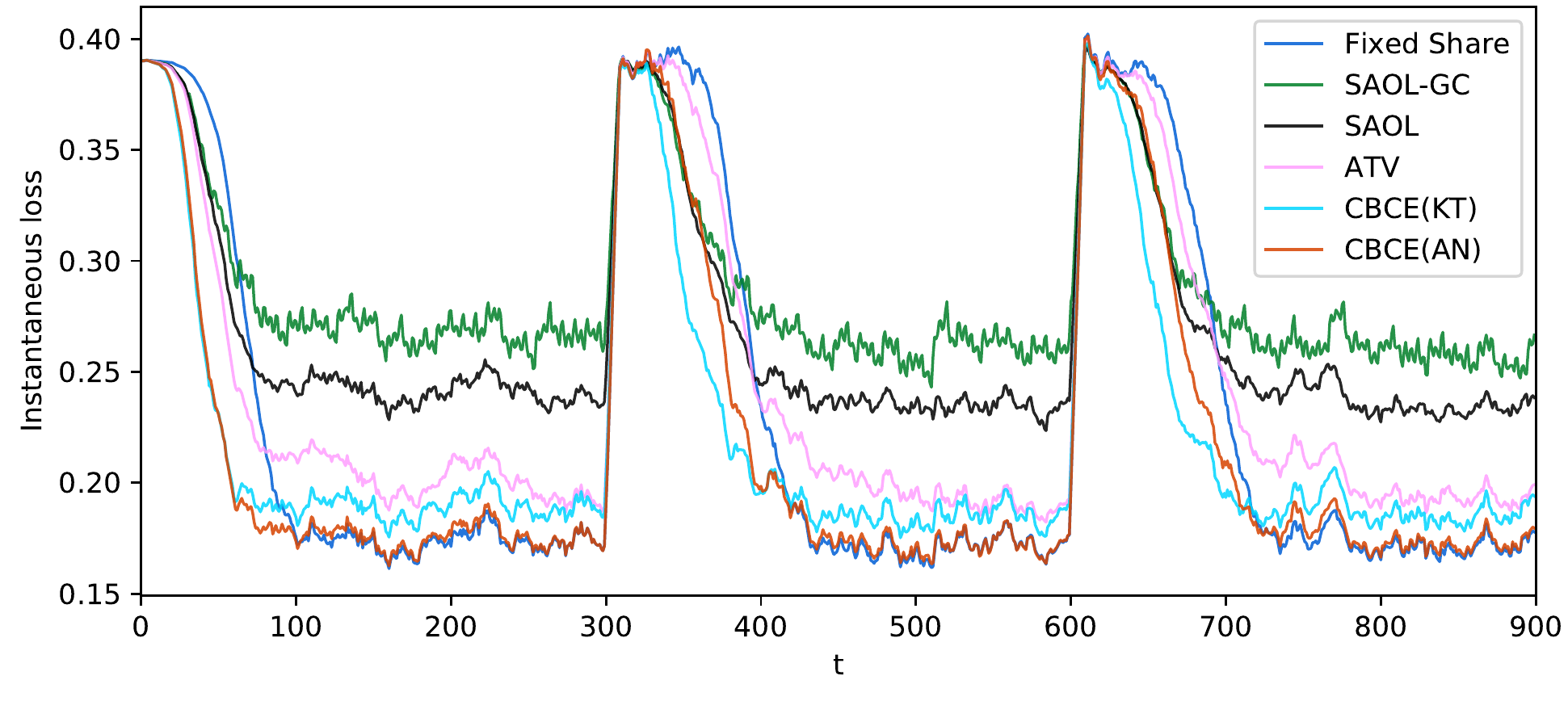} \\ (a) Learning with expert advice \\ 
    \includegraphics[width=.9\textwidth,valign=t]{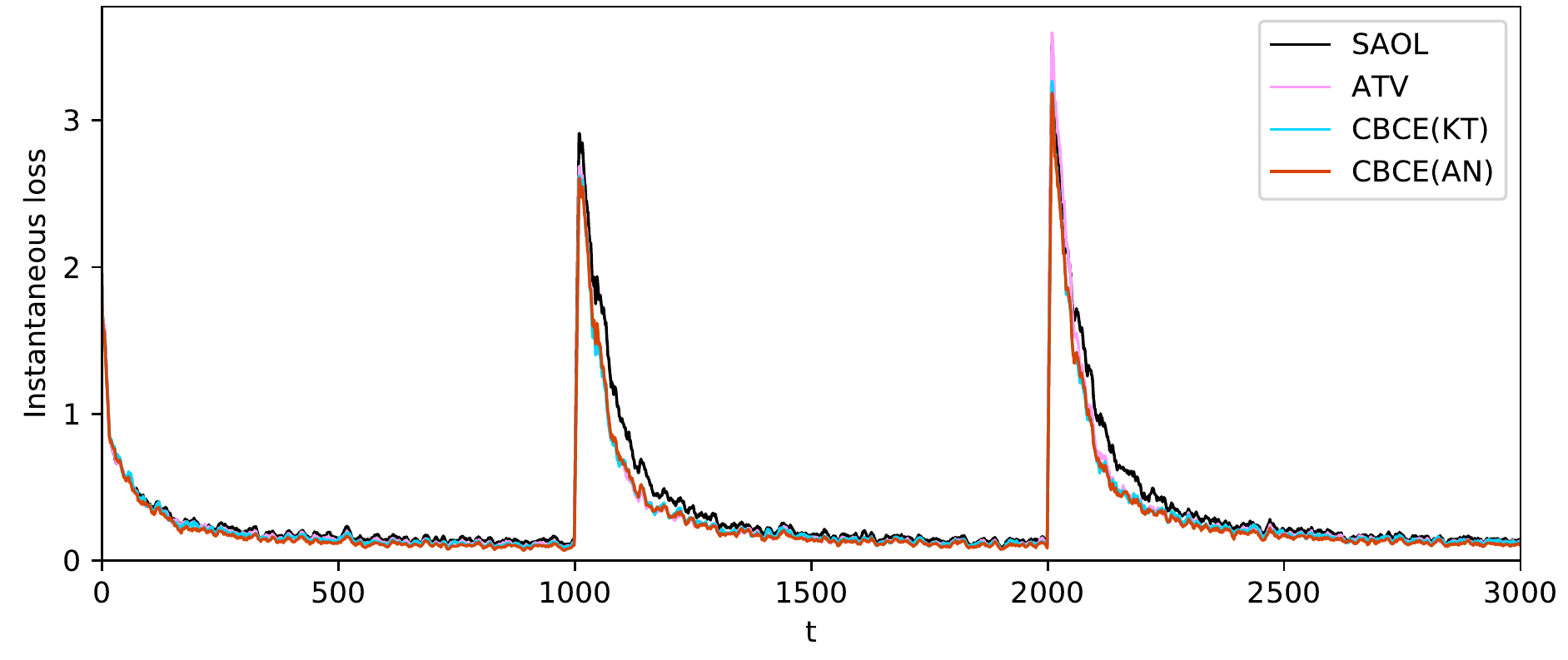} \\ (b) Metric learning\\
    \includegraphics[width=.9\textwidth,valign=t]{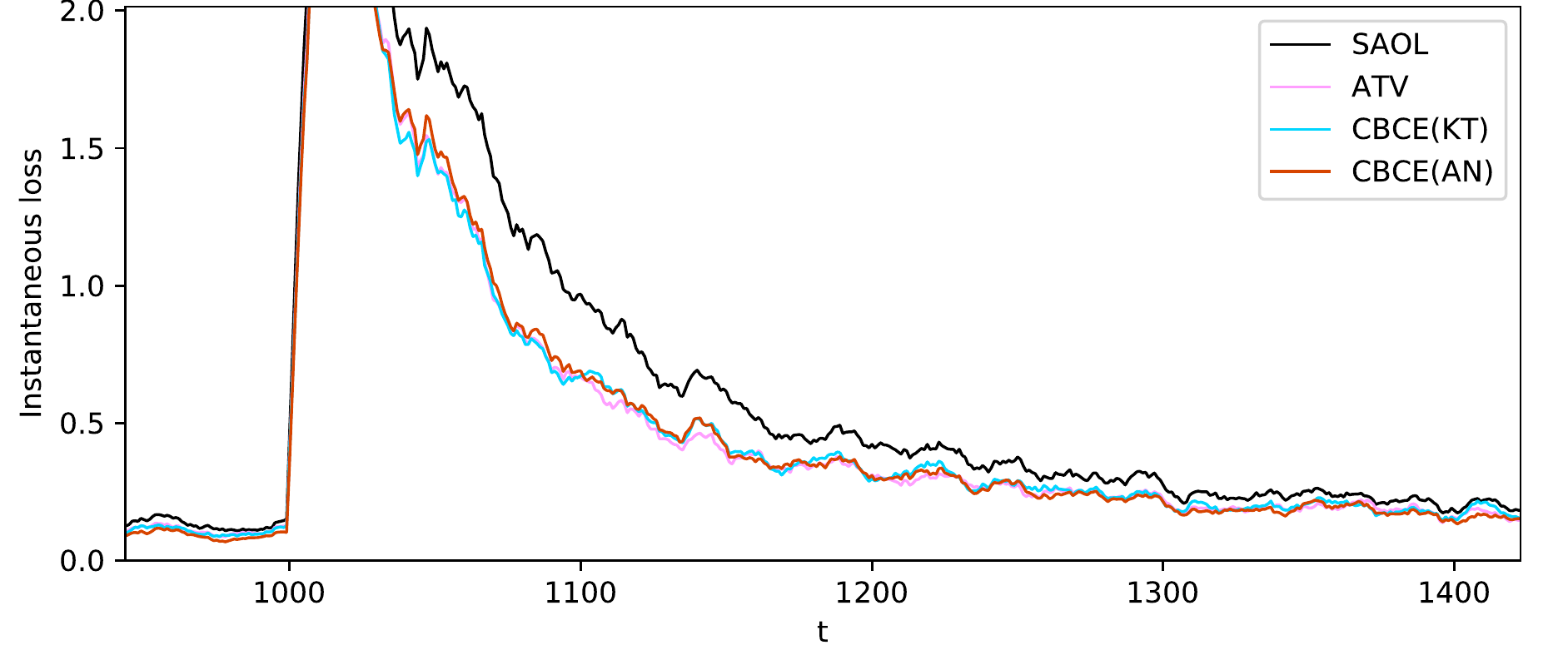} \\ (c) Metric learning (zoomed in)
  \end{tabular}
  \caption{Experiment results: Our method CBCE outperforms several baseline methods. }
  \label{fig:expr}
}
\end{figure*}

\subsection{Learning with Expert Advice (LEA)}

We consider LEA with linear loss.
That is, the loss function at time $t$ is $h_t(\p) = \bfl_t^\T \p$.
We draw linear loss $\bfl_t \in [0,1]^{N}, \forall t=1,\ldots,600$ for $N=1000$ experts by setting $\ell_{t,i}$ as the absolute value of an i.i.d. sample from $\cN(0,0.5^2)$.
Then, for time $t\in[1,300]$, we reduce loss of expert 1 by subtracting 1/2 from its loss: $\ell_{t,1} \larrow [\ell_{t,1} - 1/2]_+$.
For time $t\in[301,600]$ and $t\in[601,900]$, we perform the same for expert 2 and 3, respectively.
Thus, the best expert is 1, 2, and 3 for time segment [1,300], [301,600], and [601,900], respectively.
Finally, we cap $\ell_{t,i}$ below 1: $\ell_{t,i} \larrow \min\{\ell_{t,i},1\}$.
We use the data streaming intervals with $g=2$.
In all our experiments, DS with $g=2$ outperforms GC while spending roughly the same time.

For each meta algorithm, we use Sleeping CB with the AN potential as the black-box,\footnote{
  We also experimented with Sleeping CB with the KT potential, but we found that it works slightly worse than the AN potential in general. We omit this result here to avoid clutter.
} where $\cI_{t,i} = 1$ for all $t\ge1$ and $i\in[N]$ as there are no sleeping experts in this experiment.
We warm-start each black-box run at time $t\ge2$ by setting its prior to the decision $\p_{t-1}$ chosen by the meta algorithm at time step $t-1$.
We repeat the experiment 200 times and plot their average loss by computing moving mean with window size 10 in Figure~\ref{fig:expr}(a).
The overall winner is CBCE(AN).
While CBCE(KT) catches up with the environmental change faster than CBCE(AN), CBCE(AN) shows smaller loss than CBCE(AN) once the change settles down.
ATV is outperformed by both CBCEs but outperforms SAOL.
Note that SAOL with GC intervals (SAOL-GC) tends to incur larger loss than the SAOL with DS.
We observe that this is true for every meta algorithm, so we omit the result here to avoid clutter.
We also run Fixed Share using the parameters recommended by Corollary 5.1 of~\cite{cesa-bianchi06prediction}, which requires to know the target time horizon $T=900$ and the true number of switches $m=2$.
Such a strong assumption is often unrealistic in practice.
Furthermore, we observe that Fixed Share is the slowest in adapting to the environmental changes. 
Nevertheless, Fixed Share can be attractive since (i) after the switch has settled down its loss is competitive to CBCE(AN), and (ii) its time complexity ($O(NT)$) is lower than other algorithms ($O(NT\log T)$).

\subsection{Metric Learning}

We consider the problem of learning squared Mahalanobis distance from pairwise comparisons using the mirror descent algorithm~\cite{kunapuli12mirror}.
The data point at time $t$ is $(\z^{(1)}_t, \z^{(2)}_t, y_t)$, where 
$y_t \in\{1,-1\}$ indicates whether or not $\z^{(1)}_t \in \dsR^d$ and $\z^{(2)}_t \in \dsR^d$ belongs to the same class.
The goal is to learn a squared Mahalanobis distance parameterized by a positive semi-definite matrix $\M$ and a bias $\mu$ that have small loss $f_t([\M; \mu]) :=$
\begin{align*}
  [1 - y_t(\mu - (\z^{(1)}_t - \z^{(2)}_t)^\T \M (\z^{(1)}_t - \z^{(2)}_t))]_+ + \rho || \M ||_* \;,
\end{align*}
where $\mu$ is the bias parameter and $|| \cdot ||_*$ is the trace norm.
Such a formulation encourages predicting $y_t$ with large margin and low rank in $\M$.
A learned matrix $\M$ that has low rank can be useful in a number of machine learning tasks; e.g., distance-based classifications, clusterings, and low-dimensional embeddings.
We refer to~\cite{kunapuli12mirror} for details.

We create a scenario that exhibits shifts in the metric, which is inspired by~\cite{greenewald16nonstationary}.
Specifically, we create a mixture of three Gaussians in $\dsR^3$ whose means are well-separated, and mixture weights are .5, .3, and .2.
We draw 2000 points from it while keeping a record of their memberships.
We repeat this three times independently and concatenate these three vectors to have 2000 9-dimensional vectors.
Finally, we append to each point a 16-dimensional vector filled with Gaussian noise to have 25-dimensional vectors.
Such a construction implies that for each point there are three independent cluster memberships.
We run each algorithm for 1500 time steps.
For time 1 to 500, we randomly pick a pair of points from the data pool and assign $y_t=1$ $(y_t=-1)$ if the pair belongs to the same (different) cluster under the first clustering.
For time 501 to 1000 (1001 to 1500), we perform the same but under the second (third) clustering.
In this way, a learner must track the change in metric, especially the important low-dimensional subspaces for each time segment.

Since the loss of the metric learning is unbounded, we scale the loss by multiplying 1/5 and then capping it above at 1 as in~\cite{greenewald16nonstationary}.
Although the randomized decision discussed in Section~\ref{sec:cbce} can be used to maintain the theoretical guarantee, we stick to the weighted average since the event that the loss being capped at 1 is rare in our experiments.
As in our LEA experiment, we use the data streaming intervals with $g=2$ and initialize each black-box algorithm with the decision of the meta algorithm at the previous time step.
We repeat the experiment 200 times and plot their average loss in Figure~\ref{fig:expr}(b) by moving mean with window size 10.
While we observe that CBCE(KT), CBCE(AN), and ATV are indistinguishable (see Figure~\ref{fig:expr}(c)), all these methods outperform SAOL.
We have verified that visible gaps in Figure~\ref{fig:expr} are statistically significant.
This confirms the improved regret bound of CBCE and ATV.

\section{Future Work}
\label{sec:future}

Among a number of interesting directions, we are interested in reducing the time complexity in online learning within a changing environment.
For LEA, Fixed Share has the best time complexity.
However, Fixed Share is inherently not parameter-free; especially, it requires the knowledge of the number of shifts $m$. 
Achieving the best $m$-shift regret bound without knowing $m$ or the best SA-Regret bound in time $O(NT)$ would be an interesting future work.
The same direction is interesting for the online convex optimization (OCO) problem.
It would be interesting if an OCO algorithm such as online gradient descent can have the same SA-Regret as CBCE$\la$OGD$\ra$ without paying extra order of computation.

\section{Appendix}

\subsection{The Coin Betting Potential}
\label{sec:potential}
We precisely define the coin betting potential.
In this paper, we set $\eps=1$ throughout.
For technical reasons, we define the potential function that takes a form of $\bar{F}_t(x;y_{1:t})$.
We then define $F_t(y_{1:t}) := \bar{F}_t(\sum_{s=1}^t y_s; y_{1:t})$.
\begin{defn}\label{def:potential}
  \emph{(Coin Betting Potential~\cite{orabona16from})}
  Let $\eps>0$. 
  Let $\{\bar{F}_t\}_{t=0}^\infty$ be a sequence of functions $\bar{F}_t: (-a_t,a_t) \times [-1,1]^t \rarrow \dsR_+$ where $a_t > t$.
  The sequence $\{\bar{F}_t\}_{t=0}^\infty$ is called a \textbf{sequence of coin betting potentials for initial endowment $\eps$}, if it satisfies the following three conditions:

  \indent \emph{(a)} $\bar{F}_0(0; \cdot) = \eps$.

  \indent \emph{(b)} For every $t \ge 0$, $\bar{F}_t(x;y_{1:t})$ is even, logarithmically convex, strictly increasing on $[0,a_t)$ in the first argument, and $\lim_{x\rarrow a_t} \bar{F}_t(x;y_{1:t}) = +\infty$.
  
  \indent \emph{(c)} Define $\beta_t := \fr{\bar{F}_t(x+1; y_{1:t-1}, 1) - \bar{F}_t(x-1;y_{1:t-1}, -1)}{\bar{F}_t(x+1,y_{1:t-1}, 1) + \bar{F}_t(x-1,y_{1:t-1}, -1)}$.
  For every $t \ge 1$ every $x \in [-(t-1),(t-1)]$ and every $g\in[-1,1]$, 
  \[ 
    (1+g\beta_t) \bar{F}_{t-1}(x;y_{1:t-1}) \ge \bar{F}_t(x+g; y_{1:t-1}, g)\;.
  \]
\end{defn}

We now describe how the conditions for the coin betting potential lead to a lowerbound on the wealth: 
\[
  \Wealth_t \ge F_t\lt(  g_{1:t} \rt) 
\]
for any $g_1,g_2,\ldots,g_t \in [-1,1]$.
We use induction.
First, verify that $\Wealth_0 \ge F_0(\cdot) = \eps$, trivially.
Assuming $\Wealth_{t-1} \ge F_{t-1}(g_{1:t-1})$,
\begin{align*}
  \Wealth_t &= \Wealth_{t-1} + w_t g_t = (1 + g_t \beta_t) \Wealth_{t-1} \\
  &\ge (1 + g_t\beta_t) F_{t-1}\lt( g_{1:t-1} \rt)
  \stackrel{\text{Def.~\ref{def:potential}(c)}}{\ge} F_t\lt(g_{1:t-1}, g_t\rt) = F_t\lt( g_{1:t} \rt) \;.
\end{align*}

\subsection{SA-Regret Is Stronger Than \texorpdfstring{$m$}{}-Shift Regret}
\label{sec:saregret}

A strongly-adaptive regret bound can be turned into an $m$-shift
regret bound as follows.  Let $c>0$.  We claim that:
\begin{equation*} \begin{aligned}
  &\lt( \forall I=[I_1..I_2],\; R^\cA_I(\w) \le c\sqrt{ |I| \log (I_2)} \rt)  \\
  &\qquad\qquad \implies m\mbox{-Shift-Regret}_T^\cA \le c\sqrt{(m+1)T\log(T)} \;.
\end{aligned} \end{equation*}
To prove the claim, note that an $m$-shift sequence of experts
$\w_{1:T}$ can be partitioned into $m+1$ contiguous blocks denoted by
$I^{(1)}, \ldots, I^{(m+1)}$. For example, $(1,1,2,2,1)$ is 2-switch
sequence whose partition $\{ [1,2], [3,4], [5] \}$.  Denote by
$\w_{I^{(k)}} \in \cW$ the comparator in interval $I^{(k)}$: $ \w_t =
\w_{I^{(k)}}, \forall t\in I^{(k)}$.  Then, using Cauchy-Schwartz
inequality, we have
\begin{align} \label{regret-conversion}
  m\text{-Shift-Regret}_T^\cA  
&= \max_{\w_{1:T}: m\text{-shift seq.}} \sum_{k=1}^{m+1} R^\cA_{I^{(k)}} ( \w_{I^{(k)}} )  \notag \\
&\le \max_{\w_{1:T}: m\text{-shift seq.}} c\sqrt{\log T} \sum_{k=1}^{m+1} \sqrt{|I^{(k)}|} \notag\\
&\le \max_{\w_{1:T}: m\text{-shift seq.}} c\sqrt{\log T} \, \sqrt{m+1} \, \sqrt{\sum_{k=1}^{m+1} |I^{(k)}|} \notag \\
&= c\sqrt{\log T} \sqrt{m+1} \, \sqrt{T} \;.
\end{align}

\subsection{The Data Streaming Intervals Can Replace the Geometric Covering Intervals}
\label{sec:ds}

We show that the data streaming intervals achieves the same goal as the geometric covering intervals (GC).
Let $u(t)$ be a number such that $2^{u(t)}$ is the largest power of 2 that divides $t$; e.g., $u(12)=2$.
The data streaming intervals (DS) are 
\begin{align}\label{datastreaming}
\cJ = \{[t..(t+g\cdot2^{u(t)}-1)]: t =1,2,\ldots \} \;.
\end{align}
For any interval $J$, we denote by $J_1$ its starting time and by $J_2$ its ending time.
We say an interval $J'$ is a prefix of $J$ if $J'_1 = J_1$ and $J' \subseteq J$.

We show that DS also partitions an interval $I$ in Lemma~\ref{lem:datastreaming}.
\begin{lem}\label{lem:datastreaming}
  Consider $\cJ$ defined in~\eqref{datastreaming} with $g\ge1$.
  An interval $[I_1..I_2]\subseteq[T]$ can be partitioned to a sequence of intervals $\bar J^{(1)}, \bar J^{(2)}, \ldots, \bar J^{(n)}$ such that 
  \begin{enumerate}
    \item $\bar J^{(i)}$ is a prefix of some $J \in \cJ$ for $i=1,\ldots,n$.
    \item $|\bar J^{(i+1)}| / |\bar J^{(i)}| \ge 2$ for $i=1,\ldots,(n-1)$. 
  \end{enumerate}
\end{lem}
\begin{proof}
 For simplicity, we assume $g=1$; we later explain how the analysis can be extended to $g>1$.
 Let $I_1 = 2^u\cdot k$ where $2^u$ is the largest power of 2 that divides $I_1$.
 It follows that $k$ is an odd number.

 Let $J\in \cJ$ be the data streaming interval that starts from $I_1$.
 The length $|J|$ is $2^u$ by the definition, and $J_2$ is $I_1 + 2^u - 1$.
 Define $\bar J^{(1)} := J$.

 Then, consider the next interval $J'\in\cJ$ starting from time $I_1 + 2^u$.
 Note
 $$ J'_1 = I_1 + 2^u = 2^u \cdot k + 2^u = 2^{u+1} \cdot \fr{k+1}{2}  $$
 Note that $\fr{k+1}{2}$ is an integer since $k$ is odd.
 Therefore, $J'_1 = 2^{u'} \cdot k'$ where $u' > u$.
 It follows that the length of $J'$ is
 $$ |J'| = 2^{u'} \ge 2\cdot 2^u \;. $$
 Then, define $\bar J^{(2)} := J'$.

 We repeat this process until $I$ is completely covered by $\bar J^{(1)}, \ldots \bar J^{(n)}$ for some $n$.
 Finally, modify the last interval $\bar J^{(n)}$ to end at $I_2$ which is still a prefix of some $J\in \cJ$.
 This completes the proof for $g=1$.

 For the case of $g>1$, note that by setting $g>1$ we are only making the intervals longer.
 Observe that even if $g>1$, the sequence of intervals $\bar J^{(1)},\ldots,\bar J^{(n)}$ above are still prefixes of some intervals in $\cJ$.
\end{proof}

Note that, unlike the partition induced by GC in which interval lengths successively double then successively halve, the partition induced by DS just successively doubles its interval lengths except the last interval.
One can use DS to decompose SA-Regret of $\cM\la\cB\ra$; that is, in~\eqref{regret-decomposition}, replace $\sum_{i=-a}^b$ with $\sum_{i=1}^n$ and $J^{(i)}$ with $\bar J^{(i)}$.
Since the decomposition by DS has the same effect of ``doubling lengths', one can show that Theorem~\ref{thm:untitled} holds true with DS, too, with slightly smaller constant factors.

\subsection{A Subtle Difference between the Geometric Covering and Data Streaming Intervals}
\label{sec:subtle}

There is a subtle difference between the geometric covering intervals (GC) and the data streaming intervals (DS).

As far as the black-box algorithm has an anytime regret bound, both GC and DS can be used to prove the overall regret bound as in Theorem~\ref{thm:untitled}.
In our experiments, the blackbox algorithm has anytime regret bound, so using DS does not break the theoretical guarantee.

However, there exist algorithms with fixed-budget regret bounds only.
That is, the algorithm needs to know the target time horizon $T^*$, and the regret bound exists after exactly $T^*$ time steps only.
When these algorithms are used as the black-box, there is no easy way to prove Theorem~\ref{thm:untitled} with DS intervals.
The good news, still, is that most online learning algorithms are equipped with anytime regret bounds, and one can often use a technique called `doubling-trick'~\cite[Section~2.3]{cesa-bianchi06prediction} to turn an algorithm with a fixed budget regret into the one with an anytime regret bound.

\subsection{Technical Results}
\label{sec:technical}

\begin{lem}\label{lem:barZ_to_tilL}
  Assume $\Regret_T(\bfu) \le \sqrt{\lt( \xi + \sum_{i=1}^N u_i \barZ_{T,i}\rt) A(\bfu)}$ for some function $A$.
  Then, $\Regret_T(\bfu) = \sqrt{(\xi + 2\sum_{i=1}^N u_i \tilL_{t,i}) A(\bfu)} + A(\bfu)$.
\end{lem}
\begin{proof}
  We closely follow the proof of~\citet[Theorem 2]{luo15achieving}.
  We first claim that $\sum_{i=1}^N u_i \barZ_{t,i} \le \Regret_T(\bfu) + 2 \sum_{i=1}^N u_i \tilL_{T,i}$.
  The proof is as follows:
  \begin{align*}
      &\sum_{i=1}^N u_i \barZ_{T,i}
    \\&= \sum_i u_i \sum_{t\in[T]} | r_{t,i} | \one\{w_{t,i}>0\} + \lt|[r_{t,i}]_+\rt| \one\{ w_{t,i}\le 0 \}
    \\&= \sum_i u_i \sum_{t} | r_{t,i} | \one\{w_{t,i}>0\} + [r_{t,i}]_+ \one\{ w_{t,i}\le 0 \}
    \\&= \sum_i u_i \sum_{t} r_{t,i} \one\{w_{t,i}>0 \wedge r_{t,i} > 0\} + (-r_{t,i}) \one\{w_{t,i}>0 \wedge r_{t,i} \le 0\} 
    \\&  \qquad\qquad\qquad + [r_{t,i}]_+ \one\{ w_{t,i}\le 0 \}
    \\&= \sum_i u_i \sum_{t} \blue{ r_{t,i} \one\{w_{t,i}>0 \wedge r_{t,i} > 0\} } + (-r_{t,i}) \one\{w_{t,i}>0 \wedge r_{t,i} \le 0\} 
    \\&  \qquad\qquad\qquad \blue{\;+\; [r_{t,i}]_+ \one\{ w_{t,i}\le 0 \} }
    \\&  \qquad\qquad\qquad \blue{ \;+\;  r_{t,i} \one\{w_{t,i}>0 \wedge r_{t,i} \le 0\} } - r_{t,i} \one\{w_{t,i}>0 \wedge r_{t,i} \le 0\}  
    \\&  \qquad\qquad\qquad + [-r_{t,i}]_+\one\{ w_{t,i} \le 0 \} \blue{ \;-\; [-r_{t,i}]_+\one\{ w_{t,i} \le 0 \} }
    \\&= \mathbin{\blue{ \sum_i u_i \sum_{t}r_{t,i} }} 
    \\&  \qquad + \sum_i u_i \sum_{t} 2(-r_{t,i}) \one\{w_{t,i}>0 \wedge r_{t,i} \le 0\} + [-r_{t,i}]_+\one\{ w_{t,i} \le 0 \}
    \\&= \Regret_T(\bfu)
       + \sum_i u_i \sum_{t} 2[-r_{t,i}]_+ \one\{w_{t,i}>0\} + [-r_{t,i}]_+\one\{ w_{t,i} \le 0 \}
    \\&\le \Regret_T(\bfu) + 2 \sum_i u_i \tilL_{T,i}
  \end{align*}
  Let us simply use the notations $R$ in place of $\Regret_T(\bfu)$, $A$ in place of $A(\bfu)$, and $\tilL$ in place of $\sum_i u_i \tilL_{T,i}$.
  It is safe to assume that $R \ge 0$ since otherwise the statement of the Theorem is trivial.
  Then, by the assumption of the theorem,
  \begin{align*}
    R &\le \sqrt{(\xi + R + 2 L)A}
    \\\iff R^2 &\le (\xi + R + 2 L)A
    \\\implies R &\le \fr{1}{2}(A + \sqrt{A^2 + 4(\xi+2L)A})
    \\           &\le \fr{1}{2}(A + A + 2\sqrt{(\xi+2L)A})
    \\           &= A + \sqrt{(\xi+2L)A} \;.
  \end{align*}
\end{proof}



\section*{Acknowledgments}
This work was supported by NSF Award IIS-1447449 and NIH Award 1 U54
AI117924-01.  The authors thank Andr\'as Gy\"orgy for providing
constructive feedback and Kristjan Greenewald for providing the metric
learning code.

{\small
\bibliographystyle{icml2017_kwang}
\bibliography{library-shared}
}

\end{document}